\documentclass{article}

\usepackage{times}
\usepackage{soul}
\usepackage{url}
\usepackage[hidelinks,bookmarksopen=true]{hyperref}
\usepackage[utf8]{inputenc}
\usepackage[small]{caption}
\usepackage{graphicx}
\usepackage{amsmath}
\usepackage{amsthm}
\usepackage{booktabs}
\usepackage{algorithm}
\usepackage{algorithmic}
\newtheorem{theorem}{Theorem}

\usepackage{natbib}

\usepackage{amsmath}
\usepackage{amsthm}
\usepackage{amssymb}
\usepackage{dirtytalk}
\usepackage{subcaption}
\usepackage[capitalise,poorman]{cleveref}
\usepackage[usenames,dvipsnames]{xcolor}
\usepackage[backgroundcolor=white]{todonotes}
\newcommand{\vnorm}[1]{\left|\left|#1\right|\right|}    %
\newcommand{\tr}[1]{\mbox{tr}#1}

\newcommand{\E}{\mathbf{E}}

\newcommand{\R}{\mathbb{R}} %
\DeclareMathOperator*{\argmax}{\arg\,max}
\DeclareMathOperator*{\argmin}{\arg\,min}

\newcommand{\A}{\mathcal{A}}
\newcommand{\X}{\mathcal{X}}
\newcommand{\Y}{\mathcal{Y}}
\newcommand{\T}{^\top}
\newcommand{\La}{\Lambda}
\newcommand{\lgk}{\log(K)}
\newcommand{\prob}[1]{\Pr(#1)}
\newcommand{\expp}[1]{\exp\left(#1\right)}
\newcommand{\dmin}{\Delta_{\min}}
\newcommand{\dmint}{\Delta_{\min,t}}
\newcommand{\muh}{\hat{\mu}}
\newcommand{\Etil}{\tilde{E}}

\newcommand{\tethat}{\hat{\theta}}
\newcommand{\meanf}{h}
\newcommand{\cmin}{c_{\min}}

\usepackage{upgreek}

\usepackage{xspace}
\newcommand{\algGSE}{\ensuremath{\tt GSE}\xspace}
\newcommand{\algFW}{\ensuremath{\tt FWG}\xspace}
\newcommand{\algGSELin}{\ensuremath{\tt GSE}-\ensuremath{\tt Lin}\xspace}
\newcommand{\algGSELog}{\ensuremath{\tt GSE}-\ensuremath{\tt Log}\xspace}
\newcommand{\algGSELinFWG}{\ensuremath{\tt GSE}-\ensuremath{\tt Lin}-\ensuremath{\tt FWG}\xspace}
\newcommand{\algGSELogFWG}{\ensuremath{\tt GSE}-\ensuremath{\tt Log}-\ensuremath{\tt FWG}\xspace}
\newcommand{\algGSELinWynn}{\ensuremath{\tt GSE}-\ensuremath{\tt Lin}-\ensuremath{\tt Wynn}\xspace}
\newcommand{\algGSELinG}{\ensuremath{\tt GSE}-\ensuremath{\tt Lin}-\ensuremath{\tt Greedy}\xspace}
\newcommand{\algBG}{\ensuremath{\tt BayesGap}\xspace}
\newcommand{\algBGLin}{\ensuremath{\tt BayesGap}-\ensuremath{\tt Lin}\xspace}
\newcommand{\algBGexp}{\ensuremath{\tt BayesGap}-\ensuremath{\tt exp}\xspace}
\newcommand{\algBGM}{\ensuremath{\tt BayesGap}-\ensuremath{\tt M}\xspace}
\newcommand{\algLinUCB}{\ensuremath{\tt LinUCB}\xspace}
\newcommand{\algUCBGLM}{\ensuremath{\tt UCB}-\ensuremath{\tt GLM}\xspace}
\newcommand{\algLinGapE}{\ensuremath{\tt LinGapE}\xspace}
\newcommand{\algLinGapEG}{\ensuremath{\tt LinGapE}-\ensuremath{\tt Greedy}\xspace}
\newcommand{\algPc}{\ensuremath{\tt Peace}\xspace}
\newcommand{\algOD}{\ensuremath{\tt OD}-\ensuremath{\tt LinBAI}\xspace}

\usepackage{chngcntr}
\usepackage{apptools}
\AtAppendix{\counterwithin{lemma}{section}}
\AtAppendix{\counterwithin{theorem}{section}}
\AtAppendix{\counterwithin{cor}{section}}
\AtAppendix{\counterwithin{remark}{section}}
\newtheorem{lemma}{Lemma}[]
\newtheorem{cor}{Corollary}
\newtheorem{remark}{Remark}
\title{\bf Fixed-Budget Best-Arm Identification in Structured Bandits}

\author{%
  Mohammad Javad Azizi\\
  {\small University of Southern California}
  \\{\small \tt azizim@usc.edu}  
  \and
  Branislav Kveton\\ {\small Google Research}
  \\{\small \tt bkveton@google.com}
  \and
  Mohammad Ghavamzadeh\\{\small Google Research}
  \\{\small \tt ghavamza@google.com} 
}
\date{}
\begin{document}

\maketitle
\maketitle

\begin{abstract}
Best-arm identification (BAI) in a fixed-budget setting is a bandit problem where the learning agent maximizes the probability of identifying the optimal (best) arm after a fixed number of observations. Most works on this topic study unstructured problems with a small number of arms, which limits their applicability. We propose a general tractable algorithm that incorporates the structure, by successively eliminating suboptimal arms based on their mean reward estimates from a joint generalization model. We analyze our algorithm in linear and generalized linear models (GLMs), and propose a practical implementation based on a G-optimal design. In linear models, our algorithm has competitive error guarantees to prior works and performs at least as well empirically. In GLMs, this is the first practical algorithm with analysis for fixed-budget BAI.
\end{abstract}

\section{Introduction}
\label{sec:introduction}

\emph{Best-arm identification} (BAI) is a \textit{pure exploration} bandit problem where the goal is to identify the optimal arm. It has many applications, such as online advertising, recommender systems, and vaccine tests \citep{pmlr-v33-hoffman14,lattimore-Bandit}. In \emph{fixed-budget (FB)} BAI \citep{bubeck2010pure,audibert-2010-BAI}, the goal is to accurately identify the optimal arm within a fixed budget of observations (arm pulls). This setting is common in applications where the observations are costly. %
However, it is more complex to analyze than the \emph{fixed-confidence (FC)} setting, due to complications in budget allocation~\cite[Section 33.3]{lattimore-Bandit}. In FC BAI, the goal is to find the optimal arm with a guaranteed level of confidence, while minimizing the sample complexity. 

Structured bandits are bandit problems in which the arms share a common structure, e.g.,~\emph{linear} or \emph{generalized linear} models~\citep{GLMUCB-Filippi-2010,soare2014bestarm}. 
BAI in structured bandits has been mainly studied in the FC setting with the linear model~\citep{soare2014bestarm,xu2017fully-LinGapE,degenne2020gamification}. The literature of FB BAI for linear bandits was limited to \algBG \citep{pmlr-v33-hoffman14} for a long time. This algorithm does not explore sufficiently, and thus, performs poorly \citep{xu2017fully-LinGapE}. \cite{katzsamuels2020empirical-peace} recently proposed \algPc for FB BAI in linear bandits. Although this algorithm has desirable theoretical guarantees, it is computationally intractable, and its approximation loses the desired properties of the exact form. \algOD~\citep{yang2021minimax} is a concurrent work for FB BAI in linear bandits. It is a sequential halving algorithm with a special first stage, in which most arms are eliminated. This makes the algorithm inaccurate when the number of arms is much larger than the number of features, a common setting in structured problems. We discuss these three FB BAI algorithms in detail in \cref{sec:relatedWorks} and empirically evaluate them in \cref{sec:experiments}. 

In this paper, we address the shortcomings of prior work by developing a general successive elimination algorithm that can be applied to several FB BAI settings (\cref{sec:GSE}). The key idea is to divide the budget into multiple stages and allocate it \textit{adaptively} for exploration in each stage. As the allocation is updated in each stage, our algorithm adaptively eliminates suboptimal arms, and thus, properly addresses the important trade-off between \textit{adaptive} and \textit{static} allocation in structured BAI \citep{soare2014bestarm,xu2017fully-LinGapE}. We analyze our algorithm in \emph{linear} bandits in \cref{sec:lin}. In \cref{sec:GLM}, we extend our algorithm and analysis to \emph{generalized linear models} (GLMs) and present the first BAI algorithm for these models. Our error bounds in \cref{sec:lin,sec:GLM} motivate the use of a G-optimal allocation in each stage, for which we derive an efficient algorithm in \cref{sec:Gopt}. Using extensive experiments in \cref{sec:experiments}, we show that our algorithm performs at least as well as a number of baselines, including \algBG, \algPc, and \algOD.

\section{Problem Formulation}
\label{sec:setting}

We consider a general stochastic bandit with $K$ arms. The reward distribution of each arm $i \in \A$ (the set of $K$ arms) has mean $\mu_i$. Without loss of generality, we assume that $\mu_1 > \mu_2 \geq \cdots \geq \mu_K$; thus arm 1 is optimal. Let $x_i \in \R^d$ be the feature vector of arm $i$, such that $\sup_{i \in \A} \vnorm{x_i} \leq L$ holds, where $\vnorm{\cdot}$ is the $\ell_2$-norm in $\R^d$. We denote the observed rewards of arms by $y \in \R$. Formally, the reward of arm $i$ is $y = f(x_i) + \epsilon$, where $\epsilon$ is a $\sigma^2$-sub-Gaussian noise and $f(x_i)$ is any function of $x_i$, such that $\mu_i = f(x_i)$. In this paper, we focus on two instances of $f$: linear (\cref{eq:LinMuh}) and generalized linear (\cref{eq:GLMMuh}).

We denote by $B$ the fixed budget of arm pulls and by $\zeta$ the arm returned by the BAI algorithm. In the FB setting, the goal is to minimize the probability of error, i.e.,~$\delta = \prob{\zeta \neq 1}$~\citep{bubeck2010pure}. This is in contrast to the FC setting, where the goal is to minimize the sample complexity of the algorithm for a given upper bound on $\delta$.

\section{Generalized Successive Elimination}
\label{sec:GSE}

Successive elimination~\citep{AlmostOptimal-Karnin-2013} is a popular BAI algorithm in multi-armed bandits (MABs). Our algorithm, which we refer to as Generalized Successive Elimination (\algGSE), generalizes it to structured reward models $f$. We provide the pseudo-code of \algGSE in \cref{alg:SE}.

\algGSE operates in $s =\lceil\log_{\eta}K\rceil$ stages, where $\eta$ is a tunable elimination parameter, usually set to be 2. The budget $B$ is split evenly over $s$ stages, and thus, each stage has budget $n=\lfloor B/s\rfloor$. In each stage $t \in [s]$, \algGSE pulls arms for $n$ times and eliminates $1 - 1 / \eta$ fraction of them. We denote the set of the remaining arms at the beginning of stage $t$ by $\A_t$. By construction, only a single arm remains after $s$ stages. Thus, $\A_1 = \A$ and $\A_{s+1} = \{\zeta\}$. In stage $t$, \algGSE performs the following steps:

\noindent \textbf{Projection (Line \ref{line:project}):} To avoid singularity issues, we project the remaining arms into their spanned subspace with $d_t \leq d$ dimensions. We discuss this more after \cref{eq:LinMuh}.

\noindent \textbf{Exploration (Line \ref{line:Explore}):} The arms in $\A_t$ are sampled according to an allocation vector $\Pi_t \in \mathbb{N}^{\A_t}$, i.e.,~$\Pi_t(i)$ is the number of times that arm $i$ is pulled in stage $t$. %
In Sections~\ref{sec:lin} and~\ref{sec:GLM}, we first report our results for general $\Pi_t$ and then show how they can be improved if $\Pi_t$ is an {\em adaptive} allocation based on the G-optimal design, described in \cref{sec:Gopt}.

\noindent \textbf{Estimation (Line \ref{line:Estimate}):} Let $X_t = (X_{1, t}, \dots, X_{n, t})$ and $Y_t=(Y_{1, t}, \dots, Y_{n, t})$ be the feature vectors and rewards of the arms sampled in stage $t$, respectively. Given the reward model $f$, $X_t$, and $Y_t$, we estimate the mean reward of each arm $i$ in stage $t$, and denote it by $\muh_{i, t}$. For instance, if $f$ is a linear function, $\muh_{i, t}$ is estimated using linear regression, as in \cref{eq:LinMuh}.

\noindent \textbf{Elimination (Line \ref{line:Eliminate}):} The arms in $\A_t$ are sorted in descending order of $\muh_{i,t}$, their top $1 / \eta$ fraction is kept, and the remaining arms are eliminated.

At the end of stage $s$, only one arm remains, which is \textit{returned} as the optimal arm. While this algorithmic design is standard in MABs, it is not obvious that it would be near-optimal in structured problems, as this paper shows.

\begin{algorithm}[tb]
    \caption{\algGSE: Generalized Successive Elimination}
    \label{alg:SE}
    \textbf{Input}:  Elimination hyper-parameter $\eta$, budget $B$\\
    \textbf{Initialization}: $\;\A_1 \gets \A$, $\;t \gets 1$, $\;s\gets\lceil\log_{\eta}K\rceil$
    
    \begin{algorithmic}[1] %
    \setcounter{ALC@unique}{0}
    \WHILE{$t\leq s$}
    \STATE \textbf{Projection}: Project $\A_t$ to $d_t$ dimensions, such that $\A_t$ spans $\R^{d_t}$\label{line:project}
    
    \STATE \textbf{Exploration}: Explore $\A_t$ using the allocation $\Pi_t$ \label{line:Explore}
    
    \STATE \textbf{Estimation}: Calculate $(\muh_{i,t})_{i\in \A_t}$ based on observed $X_t$ and $Y_t$, using Eqs. (\ref{eq:LinMuh}) or (\ref{eq:GLMMuh}) \label{line:Estimate}
    
    \STATE \textbf{Elimination}: $\A_{t+1}= \underset{\substack{\A\subset\A_t: |\A|=\lceil \frac{|\A_t|}{\eta}\rceil}}{\argmax} \sum_{i\in \A}\muh_{i,t}$ \label{line:Eliminate}
    \STATE $t \gets t + 1$
    \ENDWHILE
    \STATE \textbf{Output}: $\zeta$ such that $\A_{s+1} = \{\zeta\}$\label{line:Recommendation}
    \end{algorithmic}
\end{algorithm}

\section{Linear Model}
\label{sec:lin}

We start with the linear reward model, where $\mu_i = f(x_i) = x_i\T \theta_*$, for an unknown reward parameter $\theta_* \in \R^d$. The estimate $\tethat_t$ of $\theta_*$ in stage $t$ is computed using least-squares regression as $\tethat_t=V_{t}^{-1}b_{t}$, where $V_{t}=\sum_{j=1}^{n}X_{j,t}X_{j,t}\T$ is the sample covariance matrix, and $b_{t}=\sum_{j=1}^{n}X_{j,t}Y_{j,t}$. This gives us the following mean estimate for each arm $i\in\A_t$,
\begin{align}\label{eq:LinMuh}
    \muh_{i,t}=x_i\T\tethat_t\,.
\end{align}
The matrix $V_t^{-1}$ is well-defined as long as $X_t$ spans $\R^d$. However, since \algGSE eliminates arms, it may happen that the arms in later stages do not span $\R^d$. Thus, $V_t$ could be singular and $V_t^{-1}$ would not be well-defined. We alleviate this problem by projecting\footnote{The projection can be done by multiplying the arm features with the matrix whose columns are the orthonormal basis of the subspace spanned by the arms \citep{yang2021minimax}.} the arms in $\A_t$ into their spanned subspace. We denote the dimension of this subspace by $d_t$. Alternatively, we can address the singularity issue by using the pseudo-inverse of matrices \citep{huang2021federated}. In this case, we remove the projection step, and replace $V_t^{-1}$ with its pseudo-inverse.

\subsection{Analysis}
\label{sec:Lin-Analysis}

In this section, we prove an error bound for \algGSE with the linear model. Although this error bound is a special case of that for GLMs (see \cref{thm:SE-opt-G-GLM}), we still present it because more readers are familiar with linear bandit analysis than GLMs. To reduce clutter, we assume that all logarithms have base $\eta$. We denote by $\Delta_i=\mu_1-\mu_i$, the sub-optimality gap of arm $i$, and by $\dmin = \min_{i > 1} \Delta_i$, the minimum gap, which by the assumption in \cref{sec:setting} is just $\Delta_2$.

\begin{theorem}\label{thm:SE-opt-G}
\algGSE with the linear model (\cref{eq:LinMuh}) and any valid\footnote{Allocation strategy $\Pi_t$ is valid if 
$V_t$ is invertible.
} allocation strategy $\Pi_t$ identifies the optimal arm with probability at least $1 - \delta$ for
\begin{align}\label{eq:Lin-Bd}
    \delta\leq 2\eta\lgk\exp\bigg(\frac{-\dmin^{2}\sigma^{-2}}{4\underset{i\in\A,t\in[s]}{\max} \vnorm{x_i-x_1}^2_{V_t^{-1}}} \bigg)\,.
\end{align} 
where $\vnorm{x}_{V}=\sqrt{x\T V x}$ for any $x\in\R^d$ and matrix $V\in\R^{d\times d}$. If we use the G-optimal design (\cref{alg:FWGopt}) for $\Pi_t$, then
\begin{align}\label{eq:G-opt-bnd}
    \delta\leq 2\eta\lgk\expp{\frac{-B\dmin^{2}}{4\sigma^2d\lgk }}\,.
\end{align}
\end{theorem}

We sketch the proof in \cref{subsec:proof-sketch-linear} and defer the detailed proof to \cref{app:pfs}.

The error bound in \eqref{eq:G-opt-bnd} scales as expected. Specifically, it is tighter for a larger budget $B$, which increases the statistical power of \algGSE; and a larger gap $\dmin$, which makes the optimal arm easier to identify. The bound is looser for larger $K$ and $d$, which increase with the instance size; and larger reward noise $\sigma$, which increases uncertainty and makes the problem instance harder to identify. We compare this bound to the related works in \cref{sec:relatedWorks}. 

There is no lower bound for FB BAI in structured bandits. Nevertheless, in the special case of MABs, our bound (\eqref{eq:G-opt-bnd}) matches the FB BAI lower bound $\expp{\frac{-B}{\sum_{i\in \A}\Delta_i^{-2}}}$ in~\citet{kaufmann2016complexity}, up to a factor of $\log K$. It also roughly matches the tight lower bound of~\citet{carpentier2016tight}, which is $\expp{\frac{-B}{\log(K)\sum_{i\in \A}\Delta_i^{-2}}}$. To see this, note that $\sum_{i \in \A} \Delta_i^{-2} \approx K \dmin^{-2}$ and $d = K$, when we apply \algGSE to a $K$-armed bandit problem.

\subsection{Proof Sketch}
\label{subsec:proof-sketch-linear}

The key idea in analyzing \algGSE is to control the probability of eliminating the optimal arm in each stage. Our analysis is modular and easy to extend to other elimination algorithms. %
Let $E_t$ be the event that the optimal arm is eliminated in stage $t$. Then,
$
    \delta=\prob{\cup_{t = 1}^s E_t}
  \leq \sum_{t=1}^s \prob{E_t|\bar{E}_1, \dots, \bar{E}_{t - 1}}\,,
$
where $\bar{E}_t$ is the complement of event $E_t$.
In \cref{lem:prob-elim-best}, we bound the probability that a suboptimal arm has a higher estimated mean reward than the optimal arm. This is a novel concentration result for linear bandits in successive elimination algorithms.
\begin{lemma}
\label{lem:prob-elim-best}
In \algGSE with the linear model of \cref{eq:LinMuh}, the probability that any suboptimal arm $i$ has a higher estimated mean reward than the optimal arm in stage $t$ satisfies
$
    \prob{\muh_{i,t}>\muh_{1,t}}\leq 2\exp\big(\frac{-\Delta_{i}^2\sigma^{-2}}{2\vnorm{x_i-x_1}^2_{V_t^{-1}}}\big)\,.
$
\end{lemma}
This lemma is proved using an argument mainly driven from a concentration bound. Next, we use it in \cref{lem:prob-err-one-stage-OPT-G} to bound the probability that the optimal arm is eliminated in stage $t$. %

\begin{lemma}\label{lem:prob-err-one-stage-OPT-G}
In \algGSE with the linear model (\cref{eq:LinMuh}), the probability that the optimal arm is eliminated in stage $t$ satisfies
 $
        \prob{\Etil_t}\leq 2\eta\exp\big(\frac{-\dmint^2\sigma^{-2}}{2\max_{i\in\A_t}\vnorm{x_i-x_1}^2_{V_t^{-1}}  }\big)\,,
 $
where $\dmint=\min_{i\in\A_t\backslash\{1\}}\Delta_i$ and $\Etil_t$ is a shorthand for event $E_t | \bar{E}_1, \dots, \bar{E}_{t - 1}$. 
\end{lemma}

This lemma is proved by examining how another arm can dominate the optimal arm and using Markov's inequality. Finally, we bound $\delta$ in \cref{thm:SE-opt-G} using 
a union bound. We obtain the second bound in \cref{thm:SE-opt-G} by the Kiefer-Wolfowitz Theorem \citep{kiefer_wolfowitz_1960} for the G-optimal design described in \cref{sec:Gopt}. 

\section{Generalized Linear Model}\label{sec:GLM}

We now study FB BAI in \emph{generalized linear models (GLMs)} \citep{mccullagh2019generalized}, where $\mu_i = f(x_i) = \meanf(x_i\T\theta_*)$, where $\meanf$ is a monotone function known as the \emph{mean function}. As an example, $\meanf(x)=(1+\expp{-x})^{-1}$ in logistic regression. We assume that the derivative of the mean function, ${\meanf}'$, is bounded from below, i.e.,~$\cmin\leq {\meanf}'(x_i\T \tilde{\theta}_t)$, for some $\cmin\in\R^+$ and all $i\in\A$. Here $\tilde{\theta}_t$ can be any convex combination of $\theta_*$ and its \emph{maximum likelihood estimate} $\tethat_t$ in stage $t$. This assumption is standard in GLM bandits \citep{GLMUCB-Filippi-2010,UCBGLM-LiLZ17a}. 
The existence of $\cmin$ can be guaranteed by performing forced exploration at the beginning of each stage with the sampling cost of $O(d)$ \citep{Kveton-GLM-2019}. As $\tethat_t$ satisfies $\sum_{j=1}^n \big(Y_{j,t}-\meanf(X_{j,t}\T\tethat_t)\big)X_{j,t}=0$, it can be computed efficiently by \emph{iteratively reweighted least squares} \citep{wolke1988iteratively}. This gives us the following mean estimate for each arm $i\in\A_t$,
\begin{align}\label{eq:GLMMuh}
    \hat{\mu}_{i,t}=\meanf(x_i\T\tethat_t)\,.
\end{align}
 
\subsection{Analysis}\label{sec:GLM-analysis}

In \cref{thm:SE-opt-G-GLM}, we prove similar bounds to the linear model. The proof and its sketch are presented in \cref{app:GLMpfs}. These are the first BAI error bounds for GLM bandits. 

\begin{theorem}\label{thm:SE-opt-G-GLM}
\algGSE with the GLM (\cref{eq:GLMMuh}) and any valid $\Pi_t$ identifies the optimal arm with probability at least $1 - \delta$ for
\begin{align}\label{eq:GLMeq-Bd}
\delta\leq 2\eta\lgk\exp\Bigg(\frac{-\dmin^{2}\sigma^{-2}\cmin^2}{8\underset{i\in\A,t\in[s]}{\max} \vnorm{x_i}^2_{V_t^{-1}}}\Bigg)\,.
\end{align}
If we use the G-optimal design (\cref{alg:FWGopt}) for $\Pi_t$, then
\begin{align}\label{eq:Gopt-GLM}
    \delta\leq 2\eta\lgk\expp{\frac{-B\dmin^{2}\cmin^2}{8\sigma^2 d\lgk}}\,.
\end{align}
\end{theorem}

The error bounds in \cref{thm:SE-opt-G-GLM} are similar to those in the linear model (\cref{sec:Lin-Analysis}), since
$%
  \max_{i\in\A,t\in[s]}\vnorm{x_i-x_1}_{V_t^{-1}}
  \leq 2\max_{i\in\A,t\in[s]}\vnorm{x_i}_{V_t^{-1}}\,.
$%
The only major difference is in factor $\cmin^2$, which is $1$ in the linear case. This factor arises because GLM is a linear model transformed through some non-linear mean function $h$. When $\cmin$ is small, $h$ can have flat regions, which makes the optimal arm harder to identify. Therefore, our GLM bounds become looser as $\cmin$ decreases. Note that the bounds in~\cref{thm:SE-opt-G-GLM} depend on all other quantities same as the bounds in \cref{thm:SE-opt-G} do.

The novelty in our GLM analysis is in how we control the estimation error of $\theta_*$ using our assumptions on the existence of $\cmin$. The rest of the proof follows similar steps to those in \cref{subsec:proof-sketch-linear} and are postponed to \cref{app:GLMpfs}.

\section{G-Optimal Allocation}\label{sec:Gopt} 

The stochastic error bounds in~\eqref{eq:Lin-Bd} and~\eqref{eq:GLMeq-Bd} can be optimized by minimizing  $2\max_{i\in\A,t\in[s]}\vnorm{x_i}_{V_t^{-1}}$ with respect to $V_t$, in particular, with respect to $X_t$. In each stage $t$, let $g_t(\pi, x_i) = \vnorm{x_i}^2_{V_t^{-1}}$, where $V_t=n\sum_{i\in\A_t}\pi_ix_ix_i\T$ and $\sum_{i\in\A_t}\pi_i=1$. Then, optimization of $V_t$ is equivalent to solving $\min_\pi \max_{i\in\A_t} g_t(\pi, x_i)$. This leads us to the G-optimal design \citep{kiefer_wolfowitz_1960}, which minimizes the maximum variance along all $x_i$.

We develop an algorithm based on the Frank-Wolfe (FW) method~\citep{FW-jaggi13} to find the G-optimal design. \cref{alg:FWGopt} contains the pseudo-code of  it, which we refer to as \algFW. The G-optimal design is a convex relaxation of the G-optimal allocation; an allocation is the (integer) number of samples per arm while a design is the proportion of $n$ for each arm. Defining $g_t(\pi)=\max_{i\in\A_t} g_t(\pi, x_i)$, by Danskin's theorem \citep{danskin1966theory}, we know
$
    \nabla_{\pi_j} g_t(\pi)=
    -n(x_j\T V_t^{-1}x_{\max})^2,
$
where $x_{\max}=\argmax_{i\in\A_t} g_t(\pi, x_i)$. This gives us the derivative of the objective function so we can use it in a FW algorithm. In each iteration, \algFW first minimizes the 1st-order surrogate of the objective, and then uses line search to find the best step-size and takes a gradient step. After $N$ iterations, it extracts an allocation (integral solution) from $\pi_N$ using an efficient rounding procedure from \citet{allenzhu2017nearoptimal}, which we call it $\mathrm{ROUND}(n,\pi)$. This procedure takes budget $n$, design $\pi_N$, and returns an allocation $\Pi_t$.

In \cref{app:FWGopt}, we show that the error bounds of Theorems~\ref{thm:SE-opt-G} and~\ref{thm:SE-opt-G-GLM} still hold for large enough $N$, if we use \cref{alg:FWGopt} to obtain the allocation strategy $\Pi_t$ at the exploration step (Line~\ref{line:Explore} of \cref{alg:SE}). This results in the deterministic bounds in ~\eqref{eq:G-opt-bnd} and~\eqref{eq:Gopt-GLM} in these theorems.

\begin{algorithm}[tb]
    \caption{Frank-Wolfe G-optimal allocation (\algFW)}\label{alg:FWGopt}
    \begin{algorithmic}[1]
    \setcounter{ALC@unique}{0}
        \STATE \textbf{Input}: Stage budget $n$, $\;N$ number of iterations
        \STATE \textbf{Initialization}: $\pi_0\gets(1,\dots,1)/|\A_t|\in\R^{|\A_t|}$, $\;i\gets 0$
         
         \WHILE{$i< N$}
            
            \STATE $\pi'_i\gets \argmin_{\pi':\vnorm{\pi'}_1=1} \nabla_{\pi} g_t(\pi_i)\T\pi'\;\;$ 
            \begin{small}\COMMENT{Surrogate}\end{small}
            
            \STATE $\gamma_i\gets\argmin_{\gamma\in[0,1]} g_t(\pi_i+\gamma(\pi'_i-\pi_i))\;\;$ 
            \begin{small}\COMMENT{Line search}\end{small}
            
            \STATE $\pi_{i+1}\gets\pi_i+\gamma_i(\pi'_i-\pi_i)\quad$ \begin{small}\COMMENT{Gradient step}\end{small}
            \STATE $i\gets i+1$
         \ENDWHILE
         \STATE \textbf{Output}: $\Pi_t =\mathrm{ROUND}(n,\pi_N)\;\;\;$ \begin{small}\COMMENT{Rounding}\end{small}
   \end{algorithmic}
\end{algorithm}

\section{Related Work}\label{sec:relatedWorks}

To the best of our knowledge, there is no prior work on FB BAI for GLMs and our results are the first in this setting. However, there are three related algorithms for FB BAI in linear bandits that we discuss them in detail here. Before we start, note that there is no matching upper and lower bound for FB BAI in any setting \citep{carpentier2016tight}. However, in MABs, it is known that \emph{successive elimination} is near-optimal \citep{carpentier2016tight}.

\algBG \citep{pmlr-v33-hoffman14} is a Bayesian version of the gap-based exploration algorithm in \citet{Gabillon-2012}. This algorithm models correlations of rewards using a Gaussian process. As pointed out by \citet{xu2017fully-LinGapE}, \algBG does not explore enough and thus performs poorly.
In \cref{app:BGbound}, we show under few simplifying assumptions that the error probability of \algBG is at most $KB\expp{\frac{-B\dmin^{2}}{32K}}$. Our error bound in \cref{eq:G-opt-bnd} is at most $ 2\eta\lgk\expp{\frac{-B\dmin^{2}}{4d\lgk } }$. Thus, it improves upon \algBG by reducing dependence on the number of arms $K$, from linear to logarithmic; and on budget $B$, from linear to constant. We provide a more detailed comparison of these bounds in \cref{app:BGbound}. %
Our experimental results in \cref{sec:experiments} support these observations and show that our algorithm always outperforms \algBG in the linear setting. 

\algPc \citep{katzsamuels2020empirical-peace} is mainly a FC BAI algorithm based on a transductive design, which is modified to be used in the FB setting. It minimizes the \emph{Gaussian-width} of the remaining arms with a progressively finer level of granularity. However, \algPc cannot be implemented exactly because the Gaussian width does not have a closed form and is computationally expensive to minimize. To address this, \citet{katzsamuels2020empirical-peace} proposed an approximation to \algPc, which still has some computational issues (see \cref{rem:PeaceDiffc} and \cref{sec:hard}). The error bound for \algPc, although is competitive, only holds for a relatively large budget (Theorem~7 in \citealt{katzsamuels2020empirical-peace}). We discuss this further in \cref{rem:PeaceB}. Although the comparison of their bound to ours is not straightforward, we show in \cref{app:Peace} that each bound can be superior in certain regimes that depend mainly on the relation of $d$ and $K$. 
In particular, we show two cases: (i) Based on few claims in \citet{katzsamuels2020empirical-peace} that are not rigorously proved (see (i) in \cref{app:Peace} for more details), their error bound is at most $2\lceil\log(d)\rceil\expp{\frac{-B\dmin^2}{\max_{i\in\A}\vnorm{x_i-x_1}_{V^{-1}}\log(d)}}$
which is better than our bound (\cref{eq:Lin-Bd}) only if $K> \exp(\exp(\log(d)\log\log(d)))$. (ii) We can also show that their bound is at most $2\lceil\log(d)\rceil\expp{\frac{-B\dmin^2}{d\log(K)\log(d)}}$ under the G-optimal design, which is worse than our error bound (\cref{eq:G-opt-bnd}).

In our experiments with \algPc in \cref{sec:experiments}, we implemented its approximation and it never performed better than our algorithm. We also show in \cref{sec:hard} that approximate \algPc is much more computationally expensive compared to our algorithm.

\algOD \citep{yang2021minimax} uses a G-optimal design in a sequential elimination framework for FB BAI. In the first stage, it eliminates all the arms except $\lceil d / 2\rceil$. This makes the algorithm prone to eliminating the optimal arm in the first stage, especially when the number of arms is larger than $d$. It also adds a linear (in $K$) factor to the error bound. 
In \cref{app:OD}, we provide a detailed comparison between the error bound of \algOD and ours, and show that similar to the comparison with \algPc, there are regimes where each bound is superior. However, we show that our bound is tighter in the more practically relevant setting of $K= \Omega(d^2)$. In particular, we show that their error is at most 
$\left(\frac{4K}{d}+3\log(d)\right) 
    \exp\left({\frac{(d^2-B)\Delta_{\min}^2}{32d\log(d)}}\right)
$.
Now assuming $K=d^q$ for some $q\in\R$, if we divide our bound (\cref{eq:G-opt-bnd}) with theirs, we obtain
$
    O\left(
   \frac{q\log(d)}{d^{q-1}+\log(d)}\exp\left({\frac{-d^2\Delta_{\min}^2}{d\log(d)}}\right)\right),\nonumber
$ 
which is less than 1, so in this case \emph{our error bound is tighter}. However, for $K<d(d+1)/2$, their bound is tighter. Finally, we note that our experiments in \cref{sec:experiments} and \cref{app:OD-Exp} support these observations. 

\begin{figure}[tb]
\centering
  \centering
  \includegraphics[ width=\linewidth]{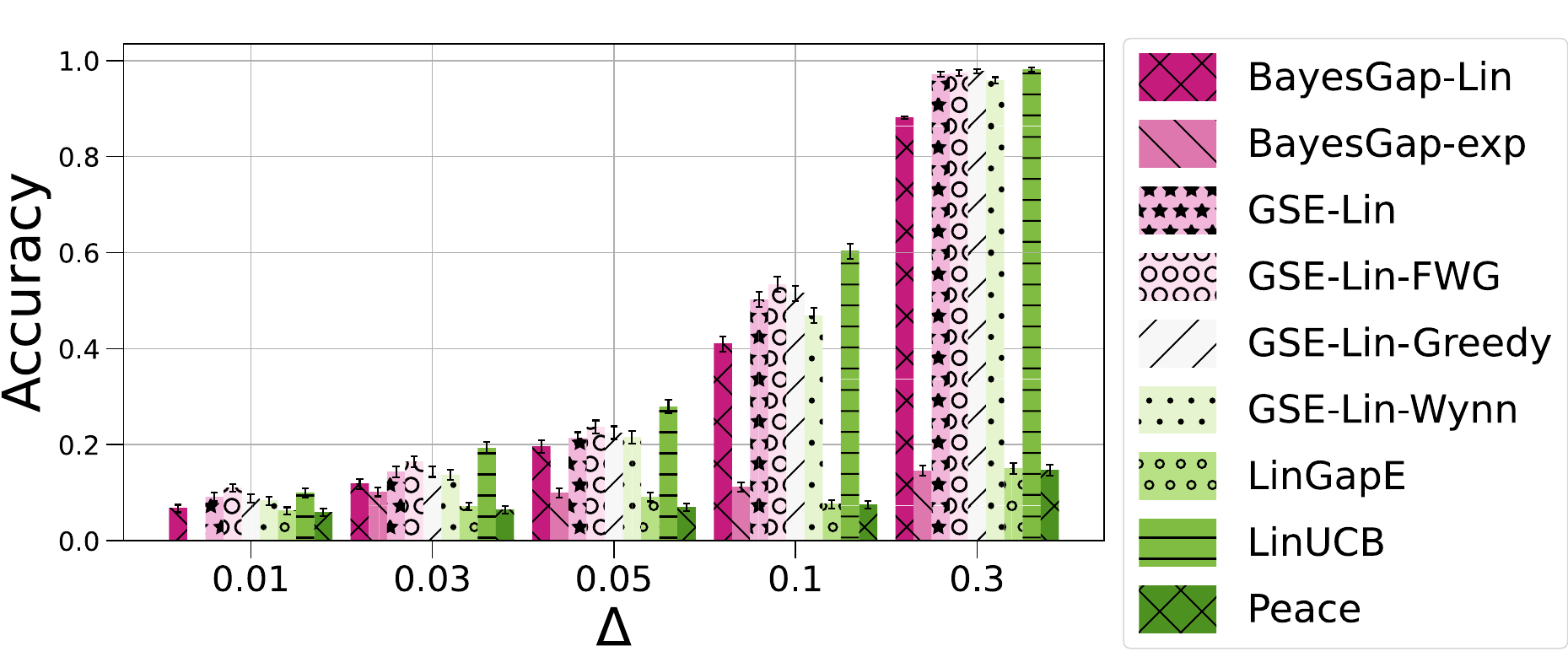}
 \caption{Static allocation.}
  \label{fig:synt5-delta}
\end{figure}

\section{Experiments}
\label{sec:experiments}

In this section, we compare \algGSE to several baselines including all linear FB BAI algorithms: \algPc, \algBG, and \algOD. Others are variants of cumulative regret (CR) bandits and FC BAI algorithms. For CR algorithms, the baseline stops at the budget limit and returns the most pulled arm.\footnote{In \cref{app:comparison}, we argue that this is a reasonable stopping rule.} We use \algLinUCB \citep{Li_2010-LinUCB} and \algUCBGLM \citep{UCBGLM-LiLZ17a}, which are the state-of-the-art for linear and GLM bandits, respectively. \algLinGapE (a FC BAI algorithm) \citep{xu2017fully-LinGapE} is used with its stopping rule at the budget limit. We tune its $\delta$ using a grid search and only report the best result. In \cref{app:variants}, we derive proper error bounds for these baselines to further justify the variants.

The \emph{accuracy} is an estimate of $1 - \delta$, as the fraction of $1000$ Monte Carlo replications where the algorithm finds the optimal arm. We run \algGSE with linear model and uniform exploration (\algGSELin), with \algFW (\algGSELinFWG), with sequential G-optimal allocation of \citet{soare2014bestarm} (\algGSELinG), and with Wynn's G-optimal method (\algGSELinWynn). For Wynn's method, see \citet{fedorov-theoryOpt}. We set $\eta = 2$ in all experiments, as this value tends to perform well in successive elimination \citep{AlmostOptimal-Karnin-2013}. For \algLinGapE, we evaluate the Greedy version (\algLinGapEG) and show its results only if it outperforms \algLinGapE. For \algLinGapEG, see \cite{xu2017fully-LinGapE}. In each experiment, we fix $K$, $B/K$, or $d$; depending on the experiment to show the desired trend. Similar trends can be observed if we fix the other parameters and change these. For further detail of our choices of kernels for \algBG and also our real-world data experiments, see \cref{app:Exp}.

\subsection{Linear Experiment: Adaptive Allocation}
\label{sec:hard}

We start with the example in \citet{soare2014bestarm}, where the arms are the canonical $d$-dimensional basis $e_1,e_2,\dots,e_d$ plus a disturbing arm $x_{d+1}=(\cos(\omega),\sin(\omega),0,\dots,0)\T$ with $\omega=1/10$. We set $\theta_*=e_1$ and $\epsilon\sim\mathcal{N}(0,10)$. Clearly the optimal arm is $e_1$, however, when the angle $\omega$ is as small as 1/10, the disturbing arm is hard to distinguish from $e_1$. As argued in \citet{soare2014bestarm}, this is a setting where an adaptive strategy is optimal (see \cref{sec:Adp.vs.Stat} for further discussion on Adaptive vs.~Static strategies).

\cref{Fig:lin-exp-1} shows that \algGSELinFWG is the second-best algorithm for smaller $K$ and the best for larger $K$. \algBGLin performs poorly here, and thus, we omit it. We \emph{conjecture} that \algBGLin fails because it uses Gaussian processes and there is a very low correlation between the arms in this experiment. \algLinGapE wins mostly for smaller $K$ and loses for larger $K$. This could be because its regret is linear in $K$ (\cref{app:comparison}). \algPc has lower accuracy than several other algorithms. We could only simulate \algPc for $K \leq 16$, since its computational cost is high for larger values of $K$. For instance, at $K = 16$, \algPc completes $100$ runs in $530$ seconds; while it only takes $7$ to $18$ seconds for the other algorithms. At $K = 32$, \algPc completes $100$ runs in $14$ hours (see \cref{app:exps}).

In this experiment, $K \approx d$ and both \algOD and \algGSE have $\log(K)$ stages and perform similarly. Therefore, we only report the results for \algGSE. This also happens in \cref{sec:staticEXp}.

\begin{figure}[tb]
\centering
  \centering
  \includegraphics[width=\textwidth]{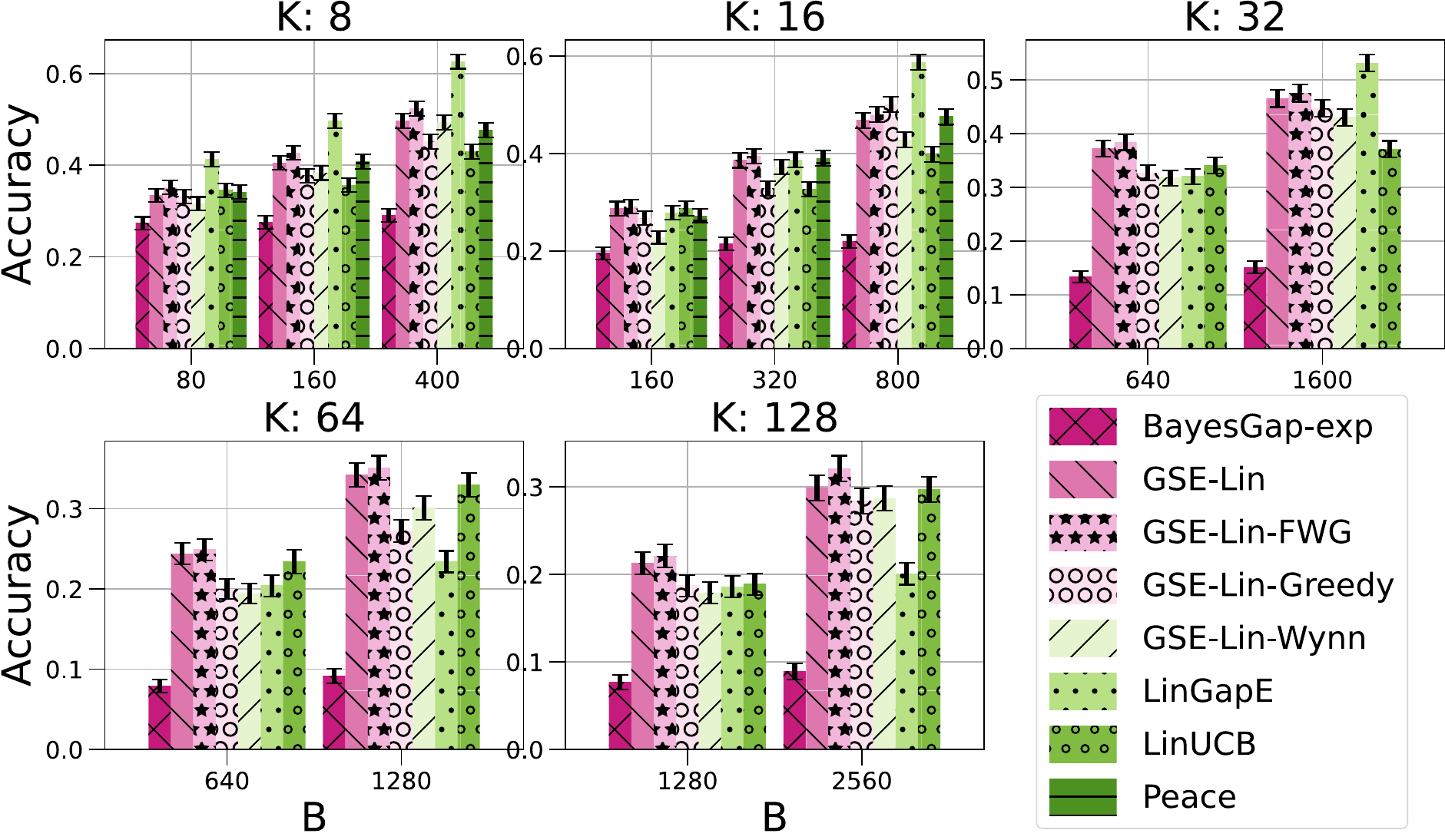}
    \caption{Adaptive instance for $d=K-1$.}\label{Fig:lin-exp-1}
\end{figure}

\subsection{Linear Experiment: Static Allocation}
\label{sec:staticEXp}
 
As in \citet{xu2017fully-LinGapE}, we take arms $e_1, e_2,...,e_{16}$ and $\theta_*=(\Delta,0,\dots,0)$, where $K=d=16$ and $B=320$. In this experiment, knowing the rewards does not change the allocation strategy. Therefore, a static allocation is optimal \citep{xu2017fully-LinGapE}. The goal is to evaluate the ability of the algorithm to adapt to a static situation.

Our results are reported in \cref{fig:synt5-delta}. We observe that \algLinUCB performs the best when $\Delta$ is small (harder instances). This is expected since suboptimal arms are well away from the optimal one, and CR algorithms do well in this case (\cref{app:comparison}). Our algorithms are the second-best when $\Delta$ is sufficiently large, converging to the optimal static allocation. \algBGexp, \algLinGapE, and \algPc cannot take advantage of larger $\Delta$, probably because they adapt to the rewards too early. This example demonstrates how well our algorithms adjust to a static allocation, and thus, properly address the tradeoff between static and adaptive allocation.

\subsection{Linear Experiment: Randomized}
\label{sec:randexp}

In this experiment, we use the example in \citet{pmlr-v80-tao18a} and \cite{yang2021minimax}. For each bandit instance, we generate i.i.d.\ arms sampled from the unit sphere centered at the origin with $d=10$. We let $\theta_*=x_i+0.01(x_j-x_i)$, where $x_i$ and $x_j$ are the two closest arms. As a consequence, $x_i$ is the optimal arm and $x_j$ is the disturbing arm. 
The goal is to evaluate the expected performance of the algorithms for a random instance to avoid bias in choosing the bandit instances.

We fix $B/K$ in \cref{fig:synt6} and compare the performance for different $K$. \algGSELinFWG has competitive performance with other algorithms. We can see that G-optimal policies have similar expected performance while \algFW is slightly better. Again, \algLinGapE performance degrades as $K$ increases and \algPc underperforms our algorithms. Moreover, the performance of \algOD worsens as $K$ increases, especially for $K>\frac{d(d+1)}{2}$. We report more experiments in this setting, comparing \algGSE to \algOD, in \cref{app:OD-Exp}. 

\begin{figure}[tb]
  \centering
  \includegraphics[ width=\linewidth]{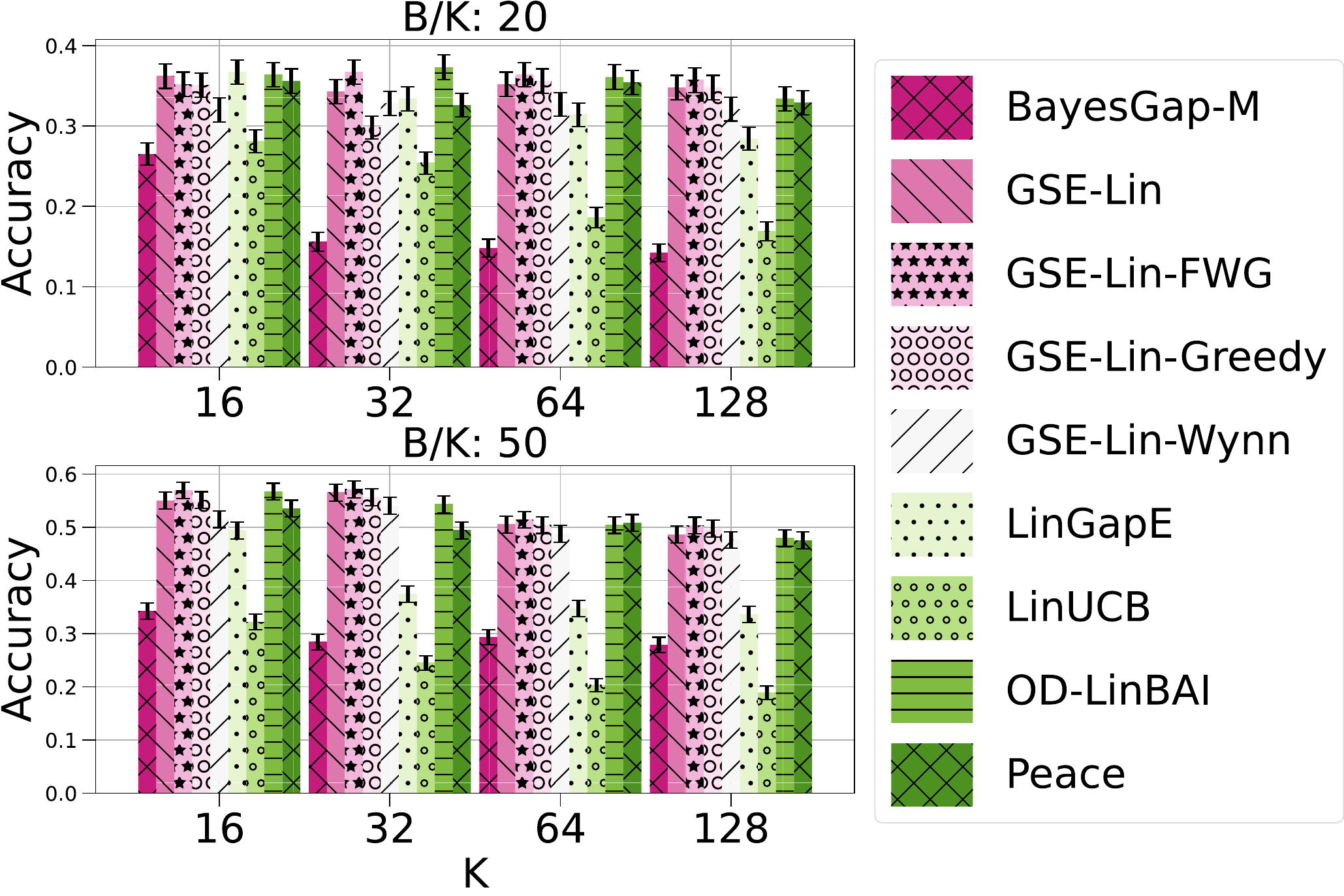}
  \caption{Randomized linear experiment.}
  \label{fig:synt6}
\end{figure}
\subsection{GLM Experiment}

As an instance of GLM, we study a \emph{logistic bandit}. We generate i.i.d.\ arms from uniform distribution on $[-0.5, 0.5]^d$ with $d \in \{5,7,10,12\}$, $K=8$, and $\theta_*\sim N(0,\frac{3}{d} I_d)$, where $I_d$ is a $d \times d$ identity matrix. The reward of arm $i$ is defined as $y_i \sim \text{Bern}(\meanf(x_i\T\theta_*))$, where $\meanf(z)=(1+\expp{-z})^{-1}$ and $\text{Bern}(z)$ is a Bernoulli distribution with mean $z$. 
We use \algGSE with a logistic regression model (\algGSELog) and also with the linear models to evaluate the robustness of \algGSE to \textit{model misspecification}. For exploration, we only use \algFW  (\algGSELogFWG), as it performs better than the other G-optimal allocations in earlier experiments. We also use a modification of \algUCBGLM \citep{UCBGLM-LiLZ17a}, a state-of-the-art GLM CR algorithm, for FB BAI.

The results in \cref{fig:logr} show \algGSE with logistic models outperforms linear models, and \algFW improves on uniform exploration in the GLM case. These experiments also show the robustness of \algGSE to model misspecification, since the linear model only slightly underperforms the logistic model. \algUCBGLM results confirm that CR algorithms could fail in BAI. \algBGM falls short for $B / K \geq 50$; the extra $B$ in their error bound also suggests failure for large $B$. In contrast, the performance of \algGSE keeps improving as $B$ increases.

\begin{figure}[tb]
\centering
  \centering
  \includegraphics[width=\textwidth]{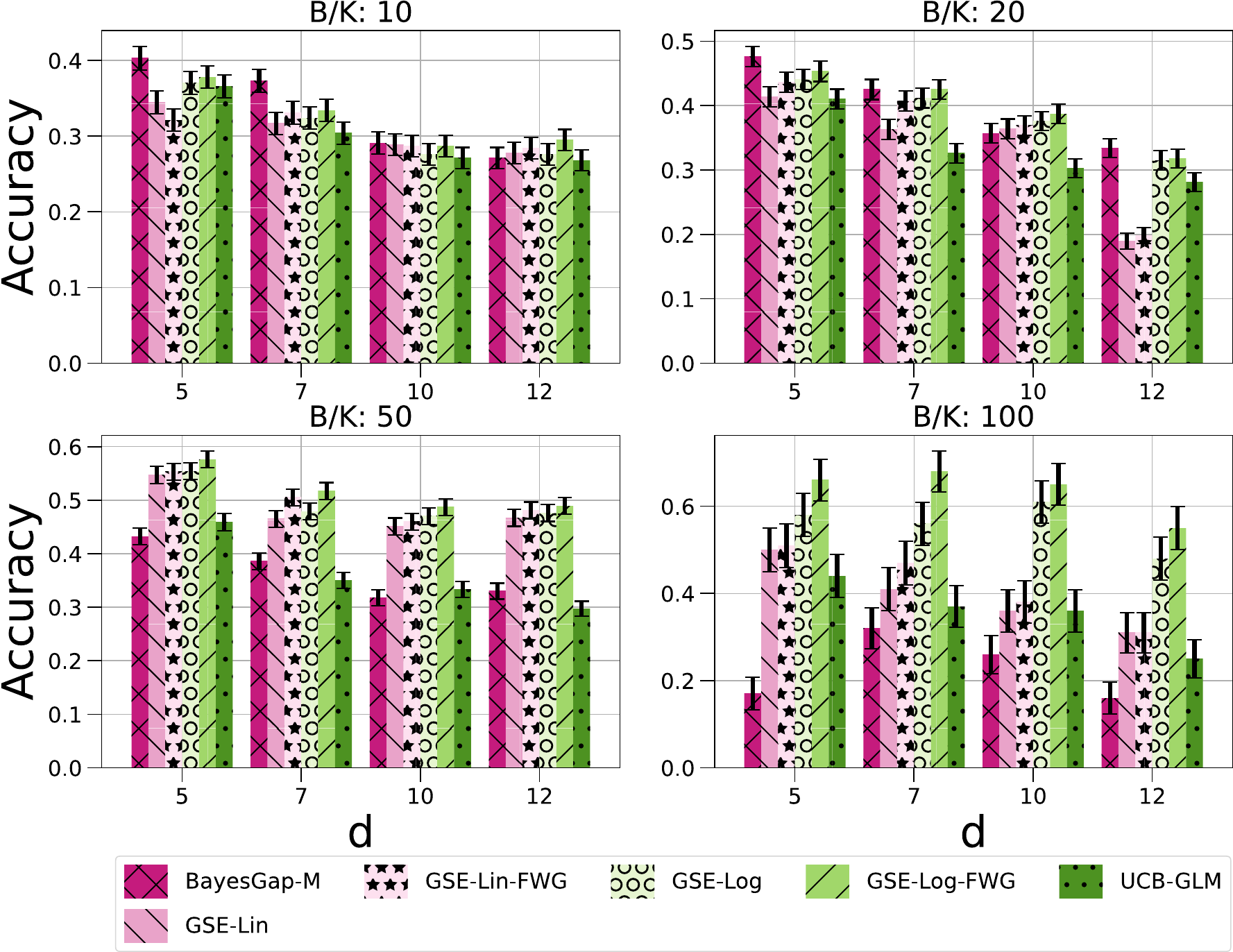}%
    \caption{Logistic bandit experiment for $K=8$.}
    \label{fig:logr}
\end{figure}

\section{Conclusions}
In this paper, we studied fixed-budget best-arm identification (BAI) in linear and generalized linear models. We proposed the \algGSE algorithm, which 
offers an adaptive framework for structured BAI. Our performance guarantees are near-optimal in MABs.
In generalized linear models, our algorithm is the first practical fixed-budget BAI algorithm with analysis. Our experiments show the efficiency and robustness (to model misspecification) of our algorithm. Extending our \algGSE algorithm to more general models could be a future direction (see \cref{app:Heur}). 

\bibliographystyle{named}
\bibliography{refs}

\clearpage
\appendix

\section{Linear Model Proofs}\label{app:pfs}

We let $K=\eta^l$ for some integer $l\geq 2$ so $s=\lgk=l$. This is for ease of reading in the analysis and proofs. We can deal with the cases in which $l$ is not an integer using rounding operators.

\begin{proof}[Proof of \cref{lem:prob-elim-best}]
Fix stage $t$. Since $t$ is fixed, we drop it in the rest of the proof. We start with
\begin{align}
\prob{\muh_{i}>\muh_{1}}&=
\prob{x_i\T\tethat>x_1\T\tethat}\\&=
\prob{(x_i-x_1)\T\tethat>0}\nonumber\\
&=\prob{(x_i-x_1)\T V^{-1}\sum_{j=1}^{n} X_{j} Y_{j}>0}\nonumber\\&=\prob{(x_i-x_1)\T V^{-1}(V\T\theta_*+\sum_{j=1}^{n}X_{j}\epsilon_{j})>0}\nonumber\\
&=
\prob{(x_i-x_1)\T V^{-1}\sum_{j=1}^{n} X_{j}\epsilon_{j}>(x_1-x_i)\T\theta_*}\label{eq:first}\;.
\end{align}

Now since $\{\epsilon_{j}\}_{j\in[n]}$ are independent, mean zero, $\sigma^2$-sub-Gaussian random variables, if we define 
\begin{align*}
    W_i&=(W_{i,1},\dots,W_{i,n})\\
    &=(x_i-x_1)\T V^{-1}\bigg(X_{1}, \cdots, X_{n}\bigg)\in \R^n\,,
\end{align*}
then by Hoeffding's inequality (Theorem 2.6.3 in \citep{Vershynin-HDP-2019}) we can write and bound \cref{eq:first} as follows
\begin{align}
\prob{\sum_{j=1}^{n} W_{i,j}\epsilon_{j}>(x_1-x_i)\T\theta_*}&\leq 2\expp{\frac{-((x_1-x_i)\T\theta_*)^2}{2\sigma^2\vnorm{W_i}_2^2}}\nonumber\\
&=2\expp{\frac{-\Delta_i^2}{2\sigma^2\vnorm{W_i}_2^2}}\label{eq:xiWi}\;.
\end{align}
It is important that the noise and features are independent. This is how we do not need adaptivity in each stage. Now since
\begin{align*}
  \vnorm{W_i}_2^2 = W_i W_i\T
  &= (x_i-x_1)\T V^{-1}V V^{-1} (x_i-x_1)\nonumber
  \\
  &=(x_i-x_1)\T V^{-1} (x_i-x_1)\;.
\end{align*} 
We can bound \cref{eq:xiWi} as follows
\begin{align*}
    2\expp{\frac{-\Delta_i^2}{2\sigma^2\vnorm{W_i}_2^2}}\leq 2\expp{\frac{-\Delta_{i}^2}{2\sigma^2\vnorm{x_i-x_1}^2_{V^{-1}}}}\;.
\end{align*}
\end{proof}

\begin{proof}[Proof of \cref{lem:prob-err-one-stage-OPT-G}]
Here we want to bound $\prob{\Etil_t}$. Let
$N_t$ denote the number of arms in $\A_t$ whose $\muh$ is larger than $\muh_1$. Then by \cref{lem:prob-elim-best}, we have
\begin{align*}
    \E(N_t)&=\sum_{i\in \A_t}\prob{\muh_{i,t}>\muh_{1,t}}
    \\
    &\leq \sum_{i\in \A_t}2\expp{\frac{-\Delta_{i}^2}{2\sigma^2\vnorm{x_i-x_1}^2_{V_t^{-1}}}}
    \\
    &\leq \sum_{i\in \A_t}2\expp{\frac{-\dmint^2}{2\sigma^2\vnorm{x_i-x_1}^2_{V_t^{-1}}}}
    \\
    &\leq 2|\A_t|\expp{\frac{-\dmint^2}{2\sigma^2\max_{i\in\A_t}\vnorm{x_i-x_1}^2_{V_t^{-1}}}}\;.
\end{align*}

Now Markov inequality gives
\begin{align*}
\prob{\Etil_t}
&=\prob{N_t\geq \frac{1}{\eta}|\A_t|}
\\
&\leq\frac{\eta\E(N_t)}{|\A_t|}
\\
&\leq 2\eta\expp{\frac{-\dmint^2}{2\sigma^2\max_{i\in\A_t}\vnorm{x_i-x_1}^2_{V_t^{-1}}  } }\;.
\end{align*}

\end{proof}

\begin{proof}[Proof of \cref{thm:SE-opt-G}]
By \cref{lem:prob-err-one-stage-OPT-G} the optimal arm is eliminated in one of the $s=\lgk$ stages with probability at most
\begin{align*}
    \delta&\leq \sum_{t=1}^s \prob{\Etil_t}\\
    &\leq2\sum_{t=1}^{s} \eta\expp{\frac{-\dmint^2}{2\sigma^2\max_{i\in\A_t}\vnorm{x_i-x_1}^2_{V_t^{-1}} } }\\
    &\leq 2\eta\lgk\expp{\frac{-\dmin^{2}}{2\sigma^2 \max_{i\in\A_t,t\in[s]}\vnorm{x_i-x_1}^2_{V_t^{-1}}} }
    \\
    &\leq 2\eta\lgk\expp{\frac{-\dmin^{2}}{4\sigma^2 \max_{i\in\A_t,t\in[s]}\vnorm{x_i}^2_{V_t^{-1}}} }\;.
\end{align*}

where we used Cauchy-Schwarz and Triangle inequality in the last inequality. Now by Kiefer-Wolfowitz Theorem \citep{kiefer_wolfowitz_1960}, we know that under the G-optimal or D-optimal design
\begin{align*}
    g_t(\pi^*)=\max_{i\in\A_t}\vnorm{x_i}^2_{V_t^{-1}}=d_t/n\leq\frac{d\lgk}{B},\quad \forall t\;.
\end{align*}
Therefore, $\max_{i\in\A_t, t\in[s]}\vnorm{x_i}^2_{V_t^{-1}}\leq d_t/n\leq\frac{d\lgk}{B}$ and 
\begin{align}
    2\eta\lgk\expp{\frac{-\dmin^{2}}{4\sigma^2 \max_{i\in\A,t\in[s]}\vnorm{x_i}^2_{V_t^{-1}}} }\nonumber\\
    \leq 2\eta\lgk\expp{\frac{-B\dmin^{2}}{4d\sigma^2 \lgk } }\label{eq:Gopt-bd}\;.
\end{align}

\end{proof}

\section{GLM Proofs}\label{app:GLMpfs}

The novelty in our GLM analysis is in how we control the estimation error of $\theta_*$ using our assumptions on the existence of $\cmin$. The rest of the proof follows similar steps to those in the linear model. The key idea is to obtain error bounds in each stage. First, in \cref{lem:prob-elim-best-GLM}, we bound the probability that a suboptimal arm has a higher estimated mean reward than the optimal arm. %

\begin{lemma}
\label{lem:prob-elim-best-GLM}
In \algGSE with the GLM of \cref{eq:GLMMuh}, the probability that any suboptimal arm $i$ has a higher estimated mean reward than the optimal arm in stage $t$ satisfies
\begin{align*}
    \prob{\muh_{i,t}&>\muh_{1,t}}\leq \\ &\expp{\frac{-\Delta_i^2\sigma^{-2}\cmin^2}{8\vnorm{x_i}_{V_t^{-1}}^2}}+\expp{\frac{-\Delta_i^2\sigma^{-2}\cmin^2}{8\vnorm{x_1}_{V_t^{-1}}^2}}.
\end{align*}
\end{lemma}

We prove this lemma using the assumption on ${h}'$ and Hoeffding's inequality. 

\begin{proof}[Proof of \cref{lem:prob-elim-best-GLM}] 
Since $t$ is fixed we drop it in the rest of this proof. Let 
\begin{align*}
    \La = \sum_{j = 1}^n {\meanf}'(X_j\T \tilde{\theta}) X_j X_j\T\;,
\end{align*}
where $\tilde{\theta}$ is some convex combination of $\theta_*$ and $\tethat$ then
\begin{align*}
    &\prob{x_i\T \tethat > x_1\T \tethat}
\\&= \prob{x_i\T \tethat - x_i\T \theta_* - \frac{\Delta_i}{2} > x_1\T \tethat - x_1\T \theta_* + \frac{\Delta_i}{2}} \\
& \leq \prob{x_i\T \tethat - x_i\T \theta_* - \frac{\Delta_i}{2} > 0} \\&\qquad\qquad + \prob{x_1\T \tethat - x_1\T \theta_* + \frac{\Delta_i}{2} < 0} \\
& = \prob{x_i\T (\tethat - \theta_*) > \frac{\Delta_i}{2}} + \prob{x_1\T (\theta_* - \tethat) > \frac{\Delta_i}{2}} \\
& = \prob{x_i\T \La^{-1} \sum_{j = 1}^n X_j \epsilon_j > \frac{\Delta_i}{2}} \\&\qquad\qquad + \prob{- x_1\T \La^{-1} \sum_{j = 1}^n X_j \epsilon_j > \frac{\Delta_i}{2}}\;.
\end{align*}
Note that we use $\Delta_i$ for the gap before the mean function transformation. In the last inequality, we used Lemma 1 in \citep{Kveton-GLM-2019}. Now if we define
\begin{align*}
    W_i=(W_{i,1},\dots,W_{i,n})&=x_i\T \La^{-1}\bigg(X_{1}, \cdots, X_{n}\bigg)\in \R^n\\\Rightarrow \vnorm{W_i}_2^2&=W_iW_i\T=x_i\T \La^{-1}V\La^{-1}x_i\;,
\end{align*}
then by Hoeffding’s inequality, we get 
\begin{align*}
    \prob{x_i\T \La^{-1} \sum_{j = 1}^n X_j \epsilon_j > \frac{\Delta_i}{2}}&=\prob{ \sum_{j = 1}^n W_{i,j} \epsilon_j > \frac{\Delta_i}{2}}\\&\leq \expp{\frac{-\Delta_i^2}{8\sigma^2\vnorm{W_i}_2^2} }\;,\\  
    \prob{- x_1\T \La^{-1} \sum_{j = 1}^n X_j \epsilon_j > \frac{\Delta_i}{2}}&=\prob{ -\sum_{j = 1}^n W_{1,j} \epsilon_j > \frac{\Delta_i}{2}}\\&\leq \expp{\frac{-\Delta_i^2}{8\sigma^2\vnorm{W_1}_2^2} }\;.
\end{align*}

Since $\tilde{\theta}$ is not known in the process, we need to dig deeper. By assumption, we know $\cmin\leq {\meanf}'(x_i\T \tilde{\theta})$ for some $\cmin\in\R^+$ and for all $i\in\A$, therefore $\cmin^{-1}V^{-1}\succeq \La^{-1}$ by definition of $\La$, and
\begin{align*}
    \vnorm{W_i}_2^2&=x_i\T \La^{-1}V\La^{-1}x_i\\&\preceq x_i\T \cmin^{-1}V^{-1}V\cmin^{-1}V^{-1}x_i\\&=\cmin^{-2} \vnorm{x_i}_{V^{-1}}^2\;,
\end{align*}
so
\begin{align*}
    &\prob{x_i\T \tethat > x_1\T \tethat}\\&\leq \expp{\frac{-\Delta_i^2\cmin^2}{8\sigma^2\vnorm{x_i}_{V^{-1}}^2} }+\expp{\frac{-\Delta_i^2\cmin^2}{8\sigma^2\vnorm{x_1}_{V^{-1}}^2}}\;.
\end{align*}
\end{proof}

Next, we bound the error probability at each stage in \cref{lem:prob-err-one-stage-OPT-GLM-G}.

\begin{lemma}\label{lem:prob-err-one-stage-OPT-GLM-G}
In \algGSE with the GLM of \cref{eq:GLMMuh}, the probability that the optimal arm is eliminated in stage $t$ satisfies
    \begin{align*}
         \prob{\Etil_t}\leq  2\eta\expp{\frac{-\dmint^2\;\sigma^{-2}\;\cmin^2}{8\;\underset{i\in\A_t}{\max}\;\vnorm{x_i}^2_{V_t^{-1}}}}\,.
    \end{align*}
\end{lemma}

\begin{proof}[Proof of \cref{lem:prob-err-one-stage-OPT-GLM-G}]
By \cref{lem:prob-elim-best-GLM} we have
\begin{align*}
    \E(N_t)&=\sum_{i\in \A_t}\prob{\muh_{i,t}>\muh_{1,t}}
    \\&\leq \sum_{i\in \A_t} \expp{\frac{-\Delta_i^2\cmin^2}{8\sigma^2\vnorm{x_i}_{V_t^{-1}}^2} }+\expp{\frac{-\Delta_i^2\cmin^2}{8\sigma^2\vnorm{x_1}_{V_t^{-1}}^2} }\\
    &\leq 2|\A_t|\expp{\frac{-\dmint^2\cmin^2}{8\sigma^2\max_{i\in\A_t}\vnorm{x_i}^2_{V_t^{-1}}}}\;.
\end{align*}
Now Markov inequality gives
\begin{align*}
\prob{\Etil_t}&=\prob{N_t>1/\eta|\A_t|}
\\
&\leq 2\eta\expp{\frac{-\dmint^2\cmin^2}{8\sigma^2\max_{i\in\A_t}\vnorm{x_i}^2_{V_t^{-1}}}}\;.
\end{align*}
\end{proof}

Finally, we bound the probability of error and conclude the proof of \cref{thm:SE-opt-G-GLM} by using this result together with a union bound and the Kiefer-Wolfowitz Theorem.

\begin{proof}[Proof of \cref{thm:SE-opt-G-GLM}]
By \cref{lem:prob-err-one-stage-OPT-GLM-G} we know that the optimal arm is eliminated in one of the $s=\lgk$ stages with a probability that satisfies
\begin{align*}
    \delta &\leq \sum_{t=1}^{s} 2\eta\expp{\frac{-\dmint^2\cmin^2}{8\sigma^2\max_{i\in\A_t}\vnorm{x_i}^2_{V_t^{-1}}}}\\
    &\leq \sum_{t=1}^{s} 2\eta\expp{\frac{-\dmin^2\cmin^2}{8\sigma^2\max_{i\in\A_t}\vnorm{x_i}^2_{V_t^{-1}}}}\\
    &\leq 2\eta\lgk\expp{\frac{-\dmin^{2}\cmin^2}{8\sigma^2 \max_{i\in\A_t,t\in[s]}\vnorm{x_i}^2_{V_t^{-1}}}}\;.
\end{align*}
Now by the Kiefer-Wolfowitz Theorem in \citep{kiefer_wolfowitz_1960}, we know $g_t(\pi^*)=\max_{i\in\A_t,t\in[s]}\vnorm{x_i}^2_{V_t^{-1}}=d_t/n\leq\frac{d\lgk}{B}$ under the G-optimal or D-optimal design, so with this design we have 
\begin{align*}
\delta&\leq 2\eta\lgk\expp{\frac{-\dmin^{2}\cmin^2}{8\sigma^2\max_{i\in\A_t,t\in[s]}\vnorm{x_i}^2_{V_t^{-1}}}}\\&\leq 2\eta\lgk\expp{\frac{-B\dmin^{2}\cmin^2}{8\sigma^2 d\lgk}}\;.
\end{align*}

\end{proof}

\section{Frank Wolfe G-optimal Design}\label{app:FWGopt}

In this section, we develop more details for our FW G-optimal algorithm.
Let $\pi:\A_t\rightarrow[0,1]$ be a distribution on $\A_t$, so $\sum_{i\in\A_t}\pi(i)=1$. For instance, based on \citet{kiefer_wolfowitz_1960} (or  Theorem 21.1 (Kiefer–Wolfowitz) and equation 21.1 from \citealt{lattimore-Bandit}) we should sample arm $i$ in stage $t$, $w_i$ times, in which 
\[w_i=\bigg\lceil \frac{\pi(i)g_t(\pi)}{\varepsilon^2}\log(1/\delta)\bigg\rceil\;,\]
where $g_t(\pi)=\max_{i\in\A}\vnorm{x_i}_{V_t^{-1}}$ and we know that $g_t(\pi^*)=d_t$ by the same Theorem. Finding $\pi^*$ is a convex problem for finite number of arms and can be solved using Frank-Wolfe algorithm (read note 3 from section 21 in \citealt{lattimore-Bandit}). After we get the optimal design $\pi^*$, we can get the optimal allocation using a randomized rounding. There are algorithms that avoid randomized rounding by starting with an allocation problem (see \cref{sec:relatedWorks}). We develop yet another efficient algorithm for the optimal design in \cref{sec:Gopt}.

Khachiyan in \citet{Khachiyan-1996} showed that if we run the FW algorithm for $O(d\log\log(K+d))$ iterations, we get $g_t(\hat{\pi})\leq d_t$, where $\hat{\pi}$ is the FW solution. More precisely, we get the following error bounds as a corollary.
\begin{cor}\label{cor:FW-Kh}
    If we use \algGSE with \algFW for Exploration, for $N=O(d\log\log(K+d))$ iterations, then
    \[\delta\leq 2\eta\lgk \expp{\frac{-B\dmin^{2}}{4d\sigma^2 \lgk} }\;.\]
\end{cor}
\begin{proof}
     As in \cref{eq:Gopt-bd} we just use $g(\hat{\pi})\leq d_t/n\leq \frac{d\log_{\eta}K}{B}$, where $\hat{\pi}$ is derived according to the algorithm in \citet{Khachiyan-1996}.
\end{proof}

Also, \citet{Kumar-Yildirim-2005} suggested an initialization of the FW algorithm which achieves a bound independent of the number of arms. In particular, we get the following corollary.
\begin{cor}
    If we use the \algGSE\ using FW algorithm to find a G-optimal design with $N=O(d\log\log(d))$ iterations starting from the initialization advised in \citet{Kumar-Yildirim-2005}, then the accuracy is lower bounded as below;
    \[\delta\leq 2\eta\lgk\expp{\frac{-B\dmin^{2}}{8d\sigma^2 \lgk } }\;.\]
\end{cor}
\begin{proof}
    By Theorem 2.3 in \citet{Ahipasaoglu2008-FWDoptimal} or note 21.2 in \citet{lattimore-Bandit} we get $g(\hat{\pi})\leq 2d/n$. The rest is same as in \cref{cor:FW-Kh}.
\end{proof}

The same sort of Corollaries hold for GLM as well by starting from \cref{thm:SE-opt-G-GLM}.

It is worth noting that based on the connection between D-optimal and G-optimal through the Kiefer-Wolfowitz Theorem \citep{kiefer_wolfowitz_1960}, we can also look at the FW algorithm for a D-optimal allocation. In D-optimal design, we seek to minimize $\det(V_t^{-1})$ which is equal to minimizing $h_t(\pi)=-\det(V_t)$
and we get
\begin{align*}
    \nabla_{\pi_i} h_t(\pi)&=
    -\det(V_t)\tr\bigg(V_t^{-1}\frac{\partial V_t}{\partial \pi_i}\bigg)\\
    &=-\det(V_t)\tr\bigg(V_t^{-1}x_ix_i\T\bigg)\;.
\end{align*}
We can use this in \algFW algorithm to implement the D-optimal allocation. The experimental results show similar performance using both G and D optimal allocation.

\section{Error Bound Comparison}\label{app:comparison}

In this section, we compare our analytical bounds to those in the related works. We try to derive similar bounds from their performance guarantee so that it is comparable to ours.

\subsection{BayesGap}\label{app:BGbound}
Based on Theorem 1 in \citet{pmlr-v33-hoffman14} we know if $\theta_*\sim\mathcal{N}(0,\eta_{b}^2I)$ then \algBG simple regret has the following bound for any $\epsilon>0$
\begin{align*}
    \prob{\mu_1-\mu_{\zeta}\geq \epsilon}\leq KB\expp{-\big(\frac{B-K}{\sigma^2}+\frac{\kappa_{b}}{\eta_{b}^2}\big)/8H_\epsilon},
\end{align*}
where $\kappa_{b}=\sum_{i\in\A}\vnorm{x_i}^{-2}$ and $H_\epsilon=\sum_{i\in\A}\big(\max(0.5(\Delta_i+\epsilon),\epsilon)\big)^{-2}$. We use the following assumptions to transform the bounds in way comparable to our bounds. Namely we use $\epsilon=0^+,\ \eta_b^2=1,\ \sigma^2=1$ where $0^+$ is a very small positive number and for simplicity we assume $\max(0.5(\Delta_i+0^+),0^+)=0.5\Delta_i$. We know that $\sum_{i\in\A}\Delta_i^{-2}\leq K\dmin^{-2},\ \kappa_b\geq KL^{-2}=K$ where $L=1$.
In this setting, the error bound satisfy the following 
\begin{align}
    \prob{\mu_1-\mu_{\zeta}> 0}
    &\simeq \prob{\mu_1-\mu_{\zeta}> 0^+ }\nonumber
    \\&\leq KB\expp{\frac{-B\dmin^{2}}{32K} }\label{eq:bd-BG}\;.
\end{align} 
This bound is comparable to our bounds in \cref{thm:SE-opt-G}.
With the same assumptions, the error (regret) bound of \algGSE with G-optimal design is
\begin{align}
    \prob{\mu_1-\mu_{\zeta}\geq 0}\leq 2\eta\lgk\expp{\frac{-B\dmin^{2}}{4d\lgk } }\label{eq:bd-SE}\;.
\end{align}

Comparing \cref{eq:bd-BG} and \cref{eq:bd-SE} we can see the improvements that \algGSE has over \algBG. In terms of $K$, we improve the linear dependence to log factors while in $B$, we improve by eliminating the linear factor to the constant outside of the exponent.

In summary, We know for \algBG, $\prob{\mu_1\geq \mu_{\zeta}}\leq KB\expp{\frac{-B\dmin^{2}}{32K} }$ while \algGSE satisfies $\prob{\mu_1\geq\mu_{\zeta}}\leq 2\eta\lgk\expp{\frac{-B\dmin^{2}}{4d\lgk } }$. The improvements in its dependence on $K$ and $B$ are obvious.

\subsection{Peace}\label{app:Peace}

We compare our error bound
\cref{eq:Lin-Bd} (or \cref{eq:G-opt-bnd}) to the performance guarantee of FB Peace in Theorem 7 of \citet{katzsamuels2020empirical-peace}. We consider different cases and claims as follows, however, it is not easy to compare our bounds to \algPc bounds since their bounds are defined using quite different approaches.

\textbf{(i)} If we accept few claims in the Peace paper (see below), their error bound is as follows:
\begin{align*}
    2\lceil\log(d)\rceil\expp{\frac{-B\dmin^2}{c\max_{i\in\A}\vnorm{x_i-x_1}_{V^{-1}}\log(d)}}\;.
\end{align*}
which is better than our error bound \cref{eq:Lin-Bd} if \boxed{\text{$K> \exp(\exp(\log(d)\log\log(d)))$}}, and worse otherwise.

\begin{proof}
    The error bound for \algPc in Theorem 7 of \citet{katzsamuels2020empirical-peace} is as follows; 
    \begin{align}\label{eq:peaceBound}
        \prob{\zeta\neq 1}\leq 2\lceil\log(\gamma(\X))\rceil\expp{\frac{-B}{c'(\rho^*+\gamma^*)\log(\gamma(\X))}}\;.
    \end{align}
    for $B\geq c\max\left((\rho^*+\gamma^*),d\right)\log(\gamma(\X))$ where $\X$ is the set of all arms $\{x_i\}_{i\in\A}$, $c'$ is a constant and
    \begin{align*}
        \gamma(X) :=\min_{\lambda\in\Delta} \E_{\xi\sim N(0,I_d)}\left(\underset{x,x'\in X}{\sup}(x-x')\T V^{-1/2}(\lambda)\xi\right)\;.
    \end{align*}
    with $\Delta$ being a probability simplex and  $X\subset\X$ is any subset of arms, $V(\lambda)=\sum_{i\in\A}\lambda_ix_ix_i\T$ and $V^{-1/2}=(V^{-1})^{1/2}$ i.e. $V^{-1/2}V^{-1/2}=V^{-1}$. Also
    \begin{align}
        \rho^* := \underset{\lambda\in\Delta}{\inf}\rho^*(\lambda)\nonumber\;,
    \end{align}
    where
    \begin{align} \rho^*:=\underset{x\in\X\backslash\{x_1\}}{\sup}\frac{\vnorm{x_1-x}^2_{V(\lambda)^{-1}}}{(\theta_*\T(x_1-x))^2}\label{eq:defRho}\;,
    \end{align}
    \begin{align*}
       \gamma^{*}&:=\inf_{\lambda \in \Delta} \gamma^{*}(\lambda)\;,
    \end{align*}
    where
    \begin{align*}
       \quad \gamma^{*}(\lambda)&:=\underset{\xi \sim N(0, I_d)}{\E}\left(\sup _{z \in \mathcal{Z} \backslash\left\{z_{*}\right\}} \frac{\left(z_{*}-z\right)^{\top} V(\lambda)^{-1 / 2} \xi}{\theta^{\top}\left(z_{*}-z\right)}\right)^{2}\;.
    \end{align*}
    
    By Proposition 1 in \citet{katzsamuels2020empirical-peace}, we know $\gamma^*\geq c\rho^*$ since $\inf_{x\neq x_1}\inf_{\lambda\in\Delta}\frac{\vnorm{x_1-x}_{V(\lambda)^{-1}}}{(\theta_*\T(x_1-x))^2}\ll\rho^*$. This claim is not proved in \citet{katzsamuels2020empirical-peace} though. As such we can give advantage to \algPc and assume the bound \cref{eq:peaceBound} is rather
    \begin{align}
        &\prob{\zeta\neq 1}\nonumber
        \\&\leq 2\lceil\log(\gamma(\X))\rceil\expp{\frac{-B}{c\rho^*\log(\gamma(\X))}}\nonumber\\
        &\leq 2\lceil\log(\gamma(\X))\rceil\expp{\frac{-B\dmin^2(\log(\gamma(\X)))^{-1}}{c\max_{i\in\A}\vnorm{x_i-x_1}_{V^{-1}}}}\nonumber\;.
    \end{align}
    where we used \cref{eq:defRho} in the last inequality. By the claim after Theorem 7 in \citet{katzsamuels2020empirical-peace} $\log(\gamma(\X))=O(\log(d))$ for linear bandits. This claim is not proved in \citet{katzsamuels2020empirical-peace} neither. Now the bound is
    \begin{align*}
        2\lceil\log(d)\rceil\expp{\frac{-B\dmin^2}{c\max_{i\in\A}\vnorm{x_i-x_1}_{V^{-1}}\log(d)}}\;.
    \end{align*}
    Bringing all the elements to the exponent, this is of order 
    \begin{align}
        \expp{\frac{-B\dmin^2}{c\max_{i\in\A}\vnorm{x_i-x_1}_{V^{-1}}\log(d)\log\log(d)}}\label{eq:Pc-1}\;,
    \end{align}
    while \cref{eq:Lin-Bd} is
    \begin{align}
        \expp{\frac{-B\dmin^2}{c\max_{i\in\A}\vnorm{x_i-x_1}_{V^{-1}}\log\log(K)}}\label{eq:GSE-1}\;,
    \end{align}
    as such the comparison boils down to comparing 
    \[\boxed{\text{$\log\log(K)~ \text{for ours vs} ~\log(d)\log\log(d)~ \text{for theirs}$}}\].
    Therefore, \cref{eq:Pc-1} is better than \cref{eq:GSE-1} if $K> \exp(\exp(\log(d)\log\log(d)))$ and worse otherwise.
\end{proof}

\textbf{(ii)} We can also show that their bound is 
\begin{align*}
 2\lceil\log(d)\rceil\expp{\frac{-B\dmin^2}{c'd\log(K)\log(d)}}\;,
\end{align*}
under the G-optimal design, which is \underline{worse than our error bound \cref{eq:G-opt-bnd}}.

\begin{proof}
    On the other hand, by definition of $\gamma$ we know 
    \begin{align}
        \gamma^*&\leq \inf_{\lambda \in \Delta} \underset{\xi \sim N(0, I)}{\E}\frac{\sup _{x \in \X \backslash\left\{x_1\right\}} \left(x_1-x\right)^{\top} V(\lambda)^{-1 / 2} \xi}{\sup_{x \in \X}\left(\theta^{\top}\left(x_1-x\right)\right)^2}\nonumber\\
        &=\gamma(\X)/\dmin^2 \label{eq:gamma} \;.
    \end{align}
    Also by Proposition 1 in \citet{katzsamuels2020empirical-peace} $\rho^*\leq \rho^*\log(K)$ so we can rewrite \cref{eq:peaceBound} as
    \begin{align*}
        \prob{\zeta\neq 1}&\leq 2\lceil\log(\gamma(\X))\rceil
        \\&\quad\expp{\frac{-B}{c'\rho^*(1+\log(K))\log(\gamma(\X))}}\\
        &\leq 2\lceil\log(d)\rceil
        \\&\quad\expp{\frac{-B\dmin^2}{c'\max_{i\in\A}\vnorm{x_i-x_1}_{V^{-1}}\log(K)\log(d)}}\;,
    \end{align*}
    which under G-optimal design is
    \begin{align*}
     2\lceil\log(d)\rceil\expp{\frac{-B\dmin^2}{c'd\log(K)\log(d)}}\;.
    \end{align*}
    Now we can compare this with \cref{eq:G-opt-bnd} and notice the extra $\log(d)$ in the exponent of \algPc bound. Again this can be written orderwise as
    \begin{align*}
     \expp{\frac{-B\dmin^2}{c'd\log(K)\log(d)\log(\log(d))}}\;.
    \end{align*}
    which by \cref{eq:GSE-1} the comparison simplifies to
    \[\boxed{\text{$\log\log(K)~ \text{for ours vs} ~d\log(K)\log(d)\log(\log(d))~ \text{for theirs}$}}\]
    which is always in the favor of our algorithm.
\end{proof}

\begin{remark}[Budget Requirement of \algPc]\label{rem:PeaceB}
    Similar to above, using Proposition 1 and the claims in \citet{katzsamuels2020empirical-peace}, for the lower bound required for \algPc budget we have
\begin{align*}
    B&\geq c \max\left(\left[\rho^{*}+\gamma^{*}\right], d\right) \log(\gamma(\mathcal{Z}))\\
    &\gtrsim c\left[\rho^*+\gamma^*\right]\log(d)\\
    &\geq c \rho^*\log(d)\\
    &=c\log(d)\frac{\max_{i\in\A}\vnorm{x_i-x_1}^2_{V^{-1}}}{\dmin^2}\;,
\end{align*}
wherein the second inequality we used the claim that $\log(\gamma(\mathcal{Z}))=O(\log(d))$ so we used $\gtrsim$ to account for that.
\end{remark}

\begin{remark}[Implementational Issues of the Approximation for Peace]\label{rem:PeaceDiffc}
    We note that the ``Computationally Efficient Algorithm for Combinatorial Bandits'' in the Peace paper is proposed for the FC setting and it is not easy to derive it for the FB setting. Nonetheless, it still has some computational issues. To name a few, consider the calculating the gradient $\nabla_{\lambda} g(\lambda;\eta;\bar{z})$ in estimateGradient subroutine, Algorithm 8 which could be cumbersome as it is done $B$ times. Also the maximization in subroutine, $\argmax_{z\in\mathcal{Z}}g(\lambda;\eta;\text{MAX-VAL};z)$ could be hard to track since the geometric properties of $\mathcal{Z}$ come into play. 
\end{remark}

\subsection{OD-LinBAI}\label{app:OD}
Here we compare the error upper bounds of our algorithm with OD-LinBAI. We show that their bound could simplify to
$$
    \Pr\left({\zeta\neq 1}\right)\leq\left(\frac{4K}{d}+3\log(d)\right) 
    \exp\left({\frac{(d^2-B)\Delta_{\min}^2}{32d\log(d)}}\right) .
$$
Now assume $K=d^q$ for some $q\in\R$, if we divide our bound \cref{eq:G-opt-bnd} with theirs we get
\begin{align}
    O\left(
   \frac{q\log(d)}{d^{q-1}+\log(d)}\exp\left({\frac{-d^2\Delta_{\min}^2}{d\log(d)}}\right)\right)\label{eq:GSE/OD}\;.
\end{align}

which is less than 1 and in this case \textbf{our error bound is tighter}. Nevertheless, in the case of $K<d(d+1)/2$, their bound is tighter.

\begin{proof}
    Our error bound for the linear case is (with $\eta=2$ and $\sigma^2=1$)
    $$
        \Pr\left({\zeta\neq 1}\right)\leq 4\log(K)\exp\left({\frac{-B\Delta_{\min}^{2}}{4d\log(K) }}\right)\;,
    $$
    while the error bound of OD-LinBAI in theorem 2 is 
    $$
    P(\zeta\neq 1)\leq \left(\frac{4K}{d}+3\log(d)\right)\exp\left(\frac{-m}{32H_{2,\text{lin}}}\right)\;,
    $$
    where $
    H_{2,\text{lin}}=\max _{2 \leq i \leq d} \frac{i}{\Delta_{i}^{2}}$ and
    \begin{align*}
        m&=\frac{B-\min \left(K, \frac{d(d+1)}{2}\right)-\sum_{r=1}^{\left\lceil\log(d)\right\rceil-1}\left[\frac{d}{2^{r}}\right\rceil}{\left\lceil\log(d)\right\rceil}
    \\\\
        &\leq
        \begin{cases}
        \frac{B-d(d+1)/2}{\log(d)}\leq \frac{B-d^2}{\log(d)} \quad & K\geq d(d+1)/2\\\\
        \frac{B-K}{\log(d)} \quad & K < d(d+1)/2
        \end{cases}\;.
    \end{align*}
    Thus their bound is
    \begin{align*}
        &\Pr\left({\zeta\neq 1}\right) 
        \\&\leq \left(\frac{4K}{d}+3\log(d)\right)\exp\left(\frac{-m}{32H_{2,\text{lin}}}\right)
        \\&\geq
         \begin{cases}
        \left(\frac{4K}{d}+3\log(d)\right) 
        \exp\left({\frac{d^2-B}{32H_{2,\text{lin}}\log(d)}}\right) \ \ & K\geq d(d+1)/2\\\\
        \left(\frac{4K}{d}+3\log(d)\right) 
        \exp\left({\frac{B-K}{32H_{2,\text{lin}}\log(d)}}\right)  \ \ & K < d(d+1)/2
        \end{cases}\;.
    \end{align*}
    
    In the case when $K\geq d(d+1)/2$ and $K$ is large, e.g. when $K=d^q$ for some $q\geq 2$, the top $d$ gaps would be approximately equal and we can roughly claim $H_{2,\text{lin}}\simeq \frac{d}{\Delta_{\min}^{2}}$. Substituting these in the above error bound of OD-LinBAI yields
    $$
        \Pr\left({\zeta\neq 1}\right)\leq\left(\frac{4K}{d}+3\log(d)\right) 
        \exp\left({\frac{(d^2-B)\Delta_{\min}^2}{32d\log(d)}}\right)\;,
    $$
    where we assumed a smaller upper bound for \algOD in its favor. Now if we divide our bound with theirs we get
    $$
    O\left(
       \frac{q\log(d)}{d^{q-1}+\log(d)}\exp\left({\frac{-d^2\Delta_{\min}^2}{d\log(d)}}\right)\right)\;,
    $$
    which is less than 1 and in this case \textbf{our error bound is tighter}. Nevertheless, in the case of $K<d(d+1)/2$, it seems their bound is tighter.
\end{proof}

\paragraph{Corner cases:}
There are special corner cases where digging into different cases of our bounds we can compare them based on \cref{eq:GSE/OD}. We can see that our error bound is worse than OD-LinBAI in a particular setting where $d$ is small, $q > 8$ is fixed, and $B \to \infty$. However, in the same setting, if $q < 8$, then the conclusion would be the opposite. We also list below several additional cases where our bound improves upon OD-LinBAI:
\begin{enumerate}
    \item $B$ is fixed, $d$ is fixed, and $q \to \infty$.

    \item $B$ is fixed, and $d, q \to \infty$.
    
    \item $B \to \infty$, while $d$ is small and $q < 8$ is fixed (the case we described above).
    
    \item $B \to \infty$ while $d \geq 3$ is fixed, and $q$ grows at the same rate as $B$. In this case, the term with $d^{q - 1}$ grows faster than the exponent, because $e^x \approx 2.7^x$.
\end{enumerate}

\section{More Experiments}\label{app:Exp}

First note that we set $\eta=2$ in the experiments and in the paper for ease of reading, but note that if we set $\eta=K$ such that the error bound in \cref{eq:G-opt-bnd} is
\begin{align*}
     \delta\leq 2K\expp{\frac{-B\dmin^{2}}{4\sigma^2d }}.
\end{align*}

which could be order-wise better than \cref{eq:G-opt-bnd} under a specific regime of $K$ and $d$. The problem is this would make the algorithm totally static and the adaptivity is lost and deteriorates the performance in instances like \cref{sec:hard}. As such, there is a trade-off and $\eta=2$ seems the best for adaptive experiments.

\subsection{Details of the Experimental Results}
\label{app:exps}

Our preliminary experiments showed that \algBG is sensitive to the choice of the kernel. Therefore, we tune \algBG in each experiment and choose the best kernel from a set of kernels. Note that this gives \algBG an advantage. The best performing kernels are linear, exponential, and Mat\'{e}rn kernels. \algBGLin stands for the linear, \algBGexp for the exponential, and \algBGM for the Mat\'{e}rn kernels.

We used a combination of computing resources. The main resource we used is the USC Center for Advanced Research Computing (https://carc.usc.edu/). Their typical compute node has dual 8 to 16 core processors and resides on a 56 gigabit FDR InfiniBand backbone, each having 16 GB memory. We also used a PC with 16 GB memory and Intel(R) Core(TM) i7-10750H CPU. For the \algPc runs we even tried Google Cloud c2-standard-60 instances with 60 CPUs and 240 GB memory.

\subsection{Real-World Data Experiments}

For this experiment, we use the "Automobile Dataset"\footnote{J. Schlimmer, Automobile Dataset, https://archive.ics.uci.edu/ml/datasets/Automobile, accessed: 01.05.2021, 1987} 
, which has features of different cars and their prices. We assume the car prices are not readily in hand, rather we get samples of them where the price of car $i$ is $N(p_i,0.1)$ where $p_i$ is the price of car $i$ in the dataset. The dataset includes $205$ cars and we use the most informative subset of features namely `curb weight`, `width`, `engine size`, `city mpg`, and `highway mpg` so $d=5$. All the features are rescaled to $[0, 1]$. We want to find the most expensive car by sampling. In each replication, we sample $K$ cars and run the algorithms with the given budget on them. The purpose is to evaluate the performance of the algorithms on a real-world dataset and test their robustness to model misspecification.

The other dataset is the "Electric Motor temperature"\footnote{J. Bocker, Electric Motor Temperature, https://www.kaggle.com/wkirgsn/electric-motor-temperature, 2019, accessed: 01.05.2021} %
and we want to find the highest engine temperature. Again, we take samples from the temperature distribution $N(\tau_i,0.1)$ where $\tau_i$ is the temperature of motor $i$ in the dataset. All the features are rescaled to [0,1] and $d$ is 11. The dataset includes $\sim998k$ data points.

\begin{figure}
\captionsetup[subfigure]{slc=off,margin={0cm,0cm}}
  \parbox[b]{.9\columnwidth}{\includegraphics[width=\hsize]{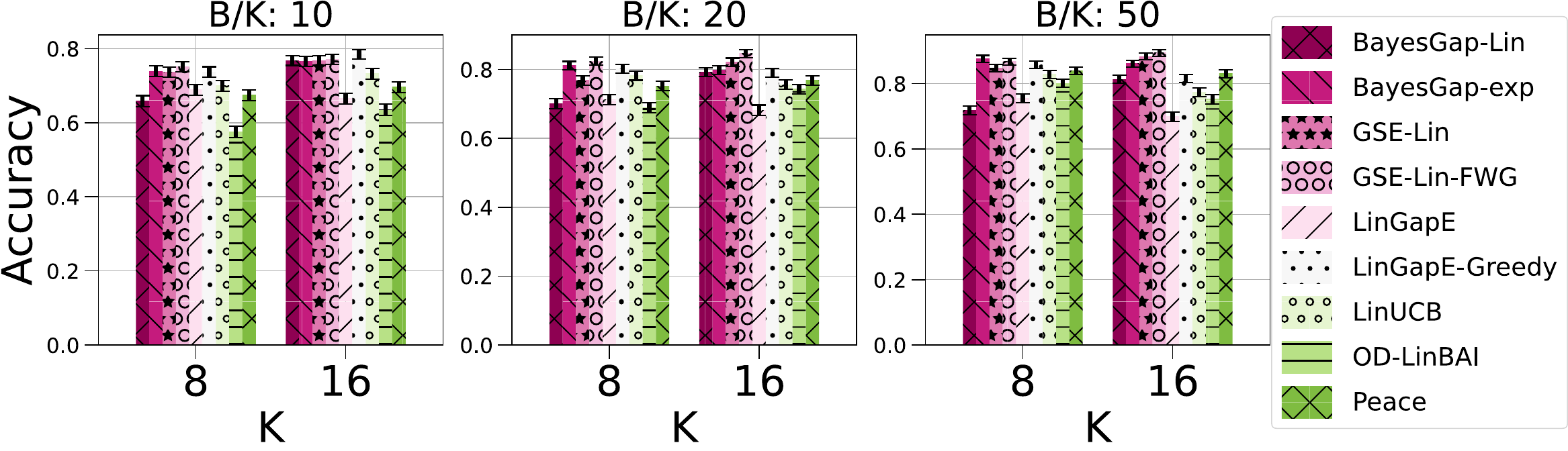}}%
  \parbox[b]{.9\columnwidth}{\subcaption{}\label{fig:auto}}\\
  \parbox[b]{.9\columnwidth}{\includegraphics[width=\hsize]{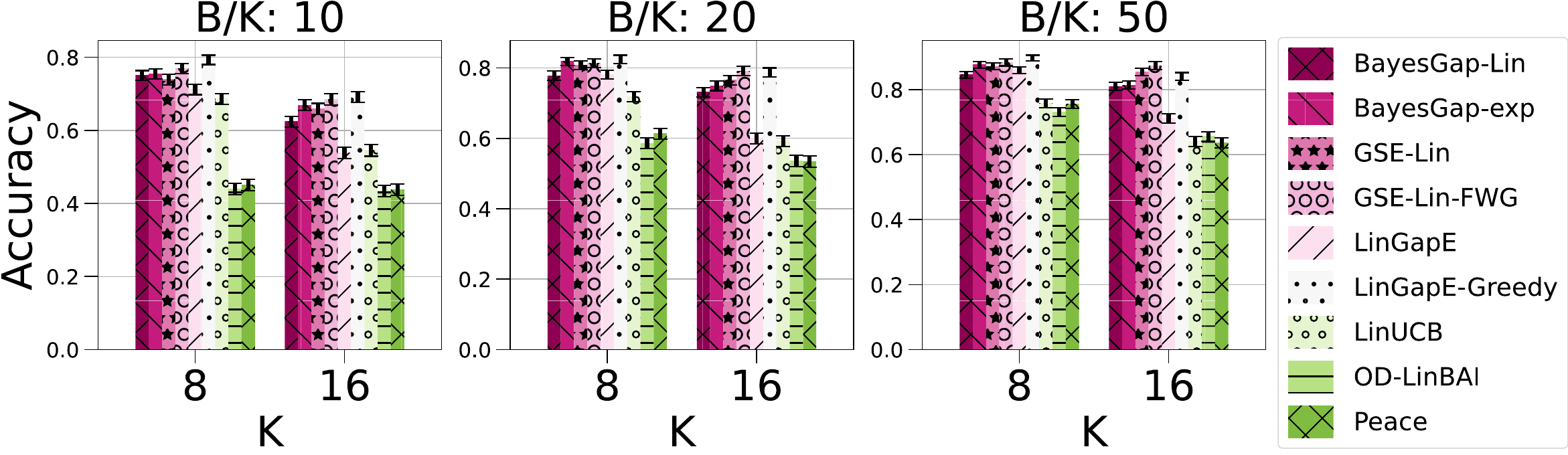}}%
  \parbox[b]{.9\columnwidth}{\subcaption{}\label{fig:pmsm}}
  \caption{Automobile (\ref{fig:auto}) and electric motor (\ref{fig:pmsm}) datasets.}
\end{figure}

\cref{fig:auto} and \cref{fig:pmsm} show the results indicating \algGSE variants outperform others or have competing performance with \algLinGapEG. The experimental results in \cref{fig:pmsm} show that our algorithm has the highest accuracy in most of the cases despite the fact that we use a linear model for a real-world dataset. This experiment also shows how well all the linear BAI algorithms could generalize to real-world situations.

\subsection{More OD-LinBAI Experiments}\label{app:OD-Exp}
In this section, we include further experiments to compare \algGSE with \algOD. First, we illustrate that our algorithm outperforms \algOD \citep{yang2021minimax}. \cref{fig:synt6-OD} shows the results for the same experiment as in \cref{sec:randexp} but with $\sigma^2=1$ (to imitate \citet{yang2021minimax}) for different $B$ and $K$. We can observe for small budgets increasing $K$ our algorithm outperforms \algOD more and more. While if $B$ is extremely large like $500$ we can see \algOD outperform \algGSE.

\begin{figure}[ht]
\centering
  \centering
  \includegraphics[height=1.1in, width=\linewidth]{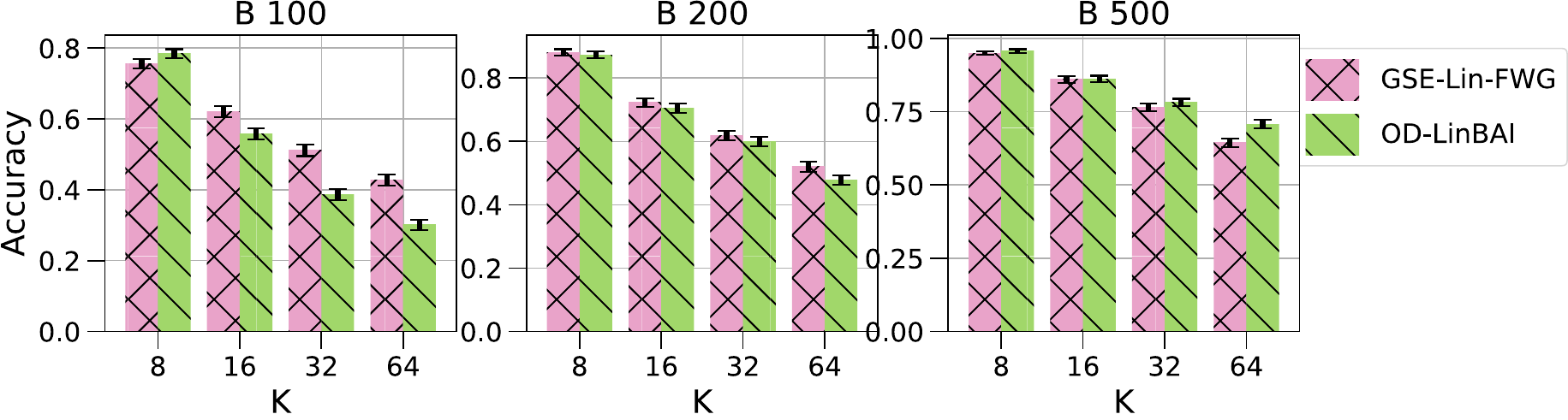}
  \caption{Randomized experiment, comparison with \algOD}
  \label{fig:synt6-OD}
\end{figure}

\cref{fig:synt7} shows the corner case experiment of Section 5.1 in \citet{yang2021minimax} where $\sigma^2=1$, $d=2$, $\theta_*=e_1=x_1$, $x_K=(\cos(3\pi/4),\sin(3\pi/4))\T$, and $x_i=(\cos(\pi/4 + \phi_i), \sin(\pi/4 + \phi_i))\T$ for $i=2,\cdots,K-1$ where $\phi_i\sim N(0,0.09^2)$ are i.i.d. samples. In this experiment \algOD outperforms our algorithm. Since \algOD only has 1 stage and simplifies to a G-optimal design, it seems that the specific setting of this experiment makes a G-optimal design most effective in only 1 stage.

\begin{figure}[ht]
\centering
  \centering
  \includegraphics[height=1.1in, width=\linewidth]{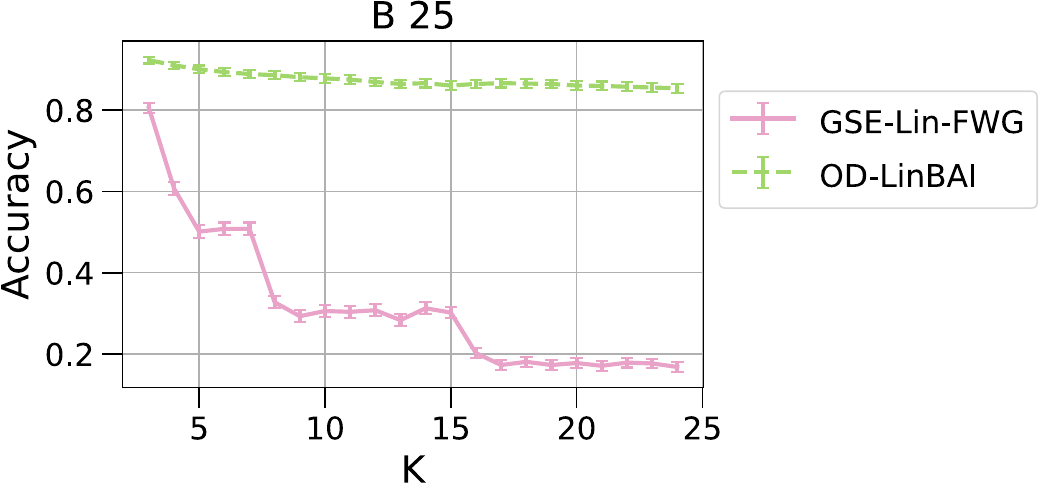}
  \caption{The corner case experiment}
  \label{fig:synt7}
\end{figure}

\section{Cumulative Regret and Fixed-Confidence Baselines}\label{app:variants}

By Proposition 33.2 in \citet{lattimore-Bandit} there is a connection between the policies designed for minimizing the cumulative regret and BAI policies. However, by the discussion after corollary 33.2 in \citet{lattimore-Bandit} the cumulative regret algorithms could under-perform for BAI depending on the bandit instance. This is because these algorithms mostly sample the optimal
arm and play suboptimal arms barely enough to ensure they are not optimal. In BAI, this leads to a highly suboptimal performance with asymptotically polynomial simple regret. However, we can compare our BAI algorithm with cumulative regret bandit algorithms as a sanity check or potential competitor. We modify the cumulative regret algorithms into a BAI using a heuristic recommendation policy, mainly by returning the most frequently played arm. 

We take \algLinUCB algorithm \citep{Li_2010-LinUCB} and let $x_t \in R^d$ be the pulled arm in stage $t$, and let the number of stages be equal $B$. \algLinUCB has the following upper bound on its expected $B$-stage regret. With probability at least $1 - \delta$,
\[x_1\T \theta_* - \left(\frac{1}{B} \sum_{t = 1}^B x_t\right)\T \!\!\!\! \theta_*
\leq \tilde{O}\left(d \sqrt{\frac{\log(1 / \delta)}{B}}\right)\;.\]
where $\tilde{O}$ hides additional log factors in $B$ \citep{Li_2010-LinUCB}. Note that this is a high-probability guarantee on the near-optimality of the average of played arms, the average feature vectors of all pulled arms. If we let this gap be smaller that $\dmin$ then we get an error bound of
\[\tilde{O}\left(\expp{\frac{-B\dmin^2}{d^2}}\right)\;.\]
In a finite arm case, like our setting, we take the most frequently played arm, call it $\varrho$, as the optimal arm. The reason is that a cumulative regret algorithm plays the potentially optimal arm the most. We conjecture this is the same as the average of played arms, since when we have large enough budget the mode (i.e. $\varrho$) converges to the mean by the Central Limit Theorem. If we set $d \sqrt{\log(1 / \delta)/B}= \dmin$ then $\delta= \expp{-B\dmin^2/d^2}$ which is of same order (except we get $d$ instead of $d^2$ which is better) as our bound in \cref{thm:SE-opt-G} also we used $\tilde{O}$ here.

For the GLM case, we employ \algUCBGLM algorithm \citep{UCBGLM-LiLZ17a} which improves the results of \textit{GLM-UCB} \citep{GLMUCB-Filippi-2010}. According to their most optimistic bound in Theorem 4 of \citet{UCBGLM-LiLZ17a}, we know the \algUCBGLM cumulative regret is less than $C\cmin\sigma^2/\kappa_l\sqrt{dB\log(B/\delta)}$ for $B$ samples, where $\kappa_l$ lower bounds the local behavior of ${\meanf}'(x)$ near $\theta_*$ and $C$ is a positive universal constant. Now since simple regret is upper bounded by the cumulative regret, this is also a simple regret bound if we return $\varrho$, i.e.
\begin{align*}
    \meanf(x_1\T\theta_*)-\meanf(x_{\varrho}\T\theta_*)&\leq \sum_{t=1}^B \meanf(x_1\T\theta_*)-\meanf(x_{t}\T\theta_*)\\&\leq C\cmin\sigma^2/\kappa_l\sqrt{dB\log(B/\delta)}\,.
\end{align*}
Now if we set this expression equal $\dmin$ we have 
\[
    \delta=B\expp{\frac{-\dmin^2\kappa_l^2}{\cmin^2dB}}\;,
\]
which is comparable to the bound in \cref{thm:SE-opt-G-GLM} but has a slower decrease in $B$.

We could also turn an FC BAI algorithm into a FB algorithm by stopping at the budget limit and using their recommendation rule. We do a grid search and use the best $\delta$ for them. \algLinGapE algorithm \citep{xu2017fully-LinGapE} is the state-of-art algorithm that performs the best among many \citep{degenne2020gamification}. By their Theorem 2, we require access to $\{\Delta_i\}_{i\in\A}$ to access the bound and find a proper $\delta$ for a given $B$. This is not very desirable from a practical point of view since we need to know $\{x_i\}_{i\in\A}$ and $\theta_*$ beforehand. As such, in our experiments, we find the best $\delta$ by a grid search based on the empirical performance of the algorithm for FB BAI. We chose the best one for them in their favor.

\section{More on Related Work}\label{app:fRelated}

\subsection{Adaptive vs.~Static BAI}\label{sec:Adp.vs.Stat} 
As argued in \citet{soare2014bestarm} and \citet{xu2017fully-LinGapE}, adaptive allocation is necessary for achieving optimal performance. However, it adds an extra $\sqrt{d}$ factor to the confidence bounds, and worsens the sample complexity. \citet{soare2014bestarm} and \citep{xu2017fully-LinGapE} propose their FC BAI algorithms $\X\Y$-adaptive and LinGapE as an attempt to address this issue. Unlike $\X\Y$-adaptive, LinGapE uses a transductive design and is shown to outperform $\X\Y$-adaptive. Our algorithm with G-optimal design applies an optimal allocation within each phase, which is different than the greedy and static allocation used by $\X\Y$-adaptive. We modify LinGapE to be applied to the FB BAI setting and compare it with our algorithm in \cref{sec:experiments}. In most cases, our algorithm performed better. Therefore, we believe our algorithm is capable of properly balancing the trade-off between adaptive and static allocations. We further discuss the related work, such as those in the FC BAI setting and optimal design in \cref{app:fRelated}.

\subsection{BAI with Successive Elimination}
Successive elimination 
\citep{AlmostOptimal-Karnin-2013} 
is common in BAI; 
\emph{Sequential Halving} from \citet{AlmostOptimal-Karnin-2013} is closely related to our work which is developed for the FB MAB problems. We extend these algorithms to the structured bandits and use an optimal static stage which adapts to the rewards between the stages.

\subsection{Optimal Design}
Finding the optimal distribution over arms is "optimal design" while finding the optimal number of samples per arm is called "optimal allocation". 
The exact optimization for many design problems is NP-Hard \citep{allenzhu2017nearoptimal}
Therefore, in the BAI literature, optimal allocation is usually treated as an implementation detail \citep{soare2014bestarm,degenne2020gamification}.
and several heuristics are proposed
\citep{Khachiyan-1996,Kumar-Yildirim-2005}. However, most of these methods are greedy \citep{degenne2020gamification}; in particular, they start with one observation per arm, and then continue with the arm that reduces the uncertainty the most in some sense \citep{soare2014bestarm}. Jedra and \citet{Jedra-2020} have a procedure (Lemma 5) that does not start with a sample for each arm but it works for FC setting. 
\citet{pmlr-v80-tao18a} suggest that solving the convex relaxation of the optimal design along with a randomized estimator yields a better solution than the greedy methods. This motivates our \algFW algorithm, which is based on FW fast-rate algorithms. 
\citet{berthet2017fast} developed a UCB Frank-Wolfe method for \emph{bandit optimization}. Their goal is to maximize an unknown smooth function with fast rates. Simple regret could be a special case of their framework, however, as of now, it is not clear how to use this method for a BAI bandit problem \citep{degenne2019nonasymptotic}. 
In our experiments, we tried several variants of optimal design techniques where the results confirm that \algFW mostly outperforms the others.

\subsection{Related Work on Fixed-Confidence Best Arm Identification}\label{app:FClit}

In this section, we discuss the works related to the FC setting in more detail. There are several BAI algorithms for linear bandits under the FC setting. We only discuss the main algorithms. \citet{soare2014bestarm} proposed the $\X\Y$ algorithms based on transductive experimental design \citep{xu2017fully-LinGapE}. In the static case, they fix all arm selections before observing any reward; as a result, it cannot estimate the near-optimal arms.   
The remedy is an algorithm that adapts the sampling based on the rewards history to assign most of the budget on distinguishing between the suboptimal arms.
\citet{Abbasi-2011} introduced a confidence bound based on Azuma’s inequality 
for adaptive strategies which is looser than the static bounds by a $\sqrt{d}$ factor. The $\X\Y$-adaptive is trying to avoid the extra $\sqrt{d}$ factor as a semi-adaptive algorithm. This algorithm improves the sample complexity in theory, but the algorithm must discard the history of previous phases to apply the bounds. This empirically degrades the performance as shown in \citet{xu2017fully-LinGapE}. 

\citet{xu2017fully-LinGapE} proposed the \algLinGapE algorithm which avoids the $\sqrt{d}$ factor by careful construction of the confidence bounds. However, the sample complexity of their algorithm is linear in $K$, which is not desirable. In this paper, we derive error bounds logarithmic in $K$ for both linear and generalized linear models.

\citet{Jedra-2020} introduced an FC BAI algorithm which tracks an optimal proportion of arm draws and updates these proportions as rarely as possible to avoid compromising its theoretical guarantees. Lemma 5 in this paper provides a procedure that can help avoid greedy sampling of all the arms.

\section{Heuristic Exploration for More General Models}\label{app:Heur}

Here we design an optimal allocation For a general model in \algGSE. Let's start with the following example; consider $K+1$ arms such that $x_i= e_1+e_1*\xi$ where $\xi\sim N(0,0.001)$ for $i\in\{1,\cdots,K\}$ and $x_{K+1}=e_2$. If we have $B=2K$, in this example, an optimal design would be to sample each $\{x_1,\dots,x_K\}$ one time and sample $x_{K+1}$, $K$ times. This is very different than a uniform exploration and motivates our idea of a generalized optimal design as follows. In each stage, cluster the remaining arms into $m$ clusters, e.g. using k-means on the $x_i$'s, then divide the budget equally between the clusters. Now in each cluster do a uniform exploration. In this way, the arms in larger clusters get a smaller budget and we get an equal amount of information in all the directions. 

For a general structured setting, we can embed the features to a lower dimension space and then apply the previous algorithms. The candidates include Principle Component Analysis, Method of Moments \citep{MOM-tripuraneni2021provable}, and encoding with Neural Networks \citep{riquelme2018deep}.

\end{document}